%% file: main.tex
\documentclass[runningheads]{llncs}

\usepackage{microtype}

\usepackage{booktabs} %

\usepackage{breakcites}
\usepackage{amsmath,amssymb}
\usepackage{graphicx,wrapfig,paralist}
\usepackage{mathtools}
\usepackage{algorithmicx,algorithm}
\usepackage[noend]{algpseudocode}
\usepackage{hyperref}
\usepackage[capitalise]{cleveref}
\usepackage{nicefrac}

\usepackage{caption}
\usepackage{listings}

\usepackage{makecell}
\usepackage{booktabs}

\usepackage{xcolor}
\usepackage{soul}

\usepackage{pgfplots}
\usepackage{pgfplotstable}
\usetikzlibrary{fadings,shapes.arrows,shadows}
\tikzfading[name=arrowfading, top color=transparent!0, bottom color=transparent!95]
\tikzset{arrowfill/.style={#1,general shadow={fill=black, shadow yshift=-0.8ex, path fading=arrowfading}}}
\tikzset{arrowstyle/.style n args={3}{draw=#2,arrowfill={#3}, single arrow,minimum height=#1, single arrow,
		single arrow head extend=.3cm,}}

\usepackage{subfig}
\usepackage{wrapfig}

\usepackage{ellipsis, mparhack, ragged2e} %
\usepackage[l2tabu, orthodox]{nag} %

\usetikzlibrary{positioning, calc, arrows.meta, shapes.misc, shapes.multipart}

\usepackage{multirow}
\newcolumntype{L}{X}
\newcolumntype{R}{>{\raggedleft\arraybackslash}X}
\newcolumntype{C}{>{\centering\arraybackslash}X}

\input{macros}

\usepackage{etoolbox} %
\newtoggle{isArxiv}
\newcommand{\ifarxivelse}[2]{\iftoggle{isArxiv}{#1}{#2}}
\toggletrue{isArxiv} %

\usepackage{marvosym}

\usepackage[firstpage]{draftwatermark}
\SetWatermarkText{\hspace*{5.0in}\raisebox{6.4in}{\includegraphics[scale=0.1]{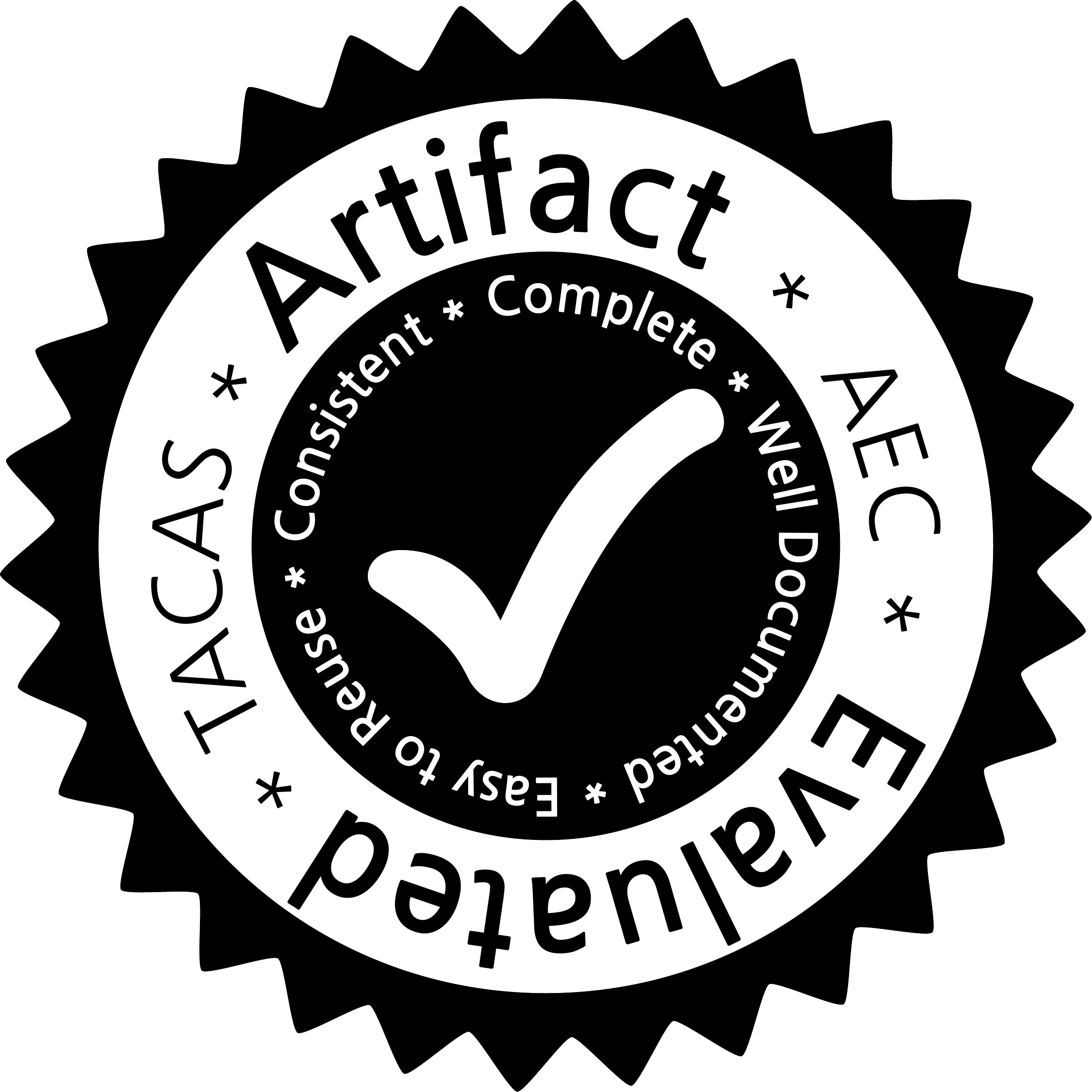}}}
\SetWatermarkAngle{0}

\begin{document}

	\titlerunning{dtControl 2.0}
	\title{dtControl 2.0: Explainable Strategy Representation via Decision Tree Learning Steered by Experts
		\thanks{%
			This work has been partially supported by the German Research Foundation (DFG) project No. 383882557 \emph{SUV} (KR 4890/2-1), No. 427755713 \emph{GOPro} (KR 4890/3-1) and the TUM International Graduate School of Science and Engineering (IGSSE) grant 10.06 \emph{PARSEC}. We thank Tim Quatman for implementing JSON-export of strategies in \storm\ and Pushpak Jagtap for his support with the \scots\ models.}
	}

	\renewcommand{\orcidID}[1]{\smash{\href{http://orcid.org/#1}{\protect\raisebox{-1.25pt}{\protect\includegraphics{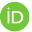}}}}}
	\author{Pranav Ashok\inst{1}\orcidID{0000-0002-1083-4741} \and Mathias Jackermeier\inst{1} \and Jan K\v{r}et\'{i}nsk\'{y}\inst{1}\orcidID{0000-0002-8122-2881} \and \\ Christoph Weinhuber\inst{1}(\Letter)\orcidID{0000-0002-1600-4933} \and Maximilian Weininger\inst{1}\orcidID{0000-0002-0163-2152} \and Mayank Yadav\inst{2}\orcidID{0000-0003-0302-8108}}
	\institute{
		Technical University of Munich, Munich, Germany \\ \email{firstname.lastname@tum.de} \and
		Department of Computer Science and Engineering, I.I.T. Delhi, \\
		New Delhi, India\\ \email{cs1180356@iitd.ac.in}
	}
	\newcommand{\shortauthors}{Ashok et al.}%
	\authorrunning{Ashok et al.}

	\maketitle
	
	\begin{abstract}
		
		Recent advances have shown how decision trees are apt data structures for concisely representing strategies (or controllers) satisfying various objectives. 
		Moreover, they also make the strategy more explainable.
		The recent tool \justdtcontrol\ had provided pipelines with tools supporting strategy synthesis for hybrid systems, such as \scots\ and \uppaal.
		We present \dtcontrol, a new version with several fundamentally novel features.
		Most importantly, the user can now provide domain knowledge to be exploited in the decision tree learning process and can also interactively steer the process based on the dynamically provided information.
		To this end, we also provide a graphical user interface. It allows for inspection and re-computation of parts of the result, suggesting as well as receiving advice on predicates, and visual simulation of the decision-making process.
		Besides, we interface model checkers of probabilistic systems, namely \storm\ and \prism\ and provide dedicated support for categorical enumeration-type state variables.
		Consequently, the controllers are more explainable and smaller.
		
	\end{abstract}

	\input{1_intro}
	\input{2_preliminaries}

	\input{3_tool}

	\input{4_theory}

	\input{5_experiments}

	\section{Conclusion}\label{sec:conc}

	We have presented a radically new version of the tool \tool\ for representing controllers by decision trees.
	The tool now features a graphical user interface, allowing both experts and non-experts to conveniently interact with the decision tree learning process as well as the resulting tree.
	There is now a range of possibilities on how the user can provide additional information.
	The algebraic predicates provide the means to capture the (often non-linear) relationships from the domain knowledge.
	The categorical predicates together with the interface to probabilistic model checkers allow for efficient representation of strategies for Markov decision processes, too.
	Finally, the more efficient determinization yields very small (possibly non-performant) controllers, which are particularly useful for debugging the model.

	We see at least two major promising future directions.
	Firstly, synthesis of predicates could be made more automatic using mathematical reasoning on the domain knowledge, such as substituting expressions with a certain unit of measurement into other domain equations in the places with the same unit of measurement, e.g. to plug difference of two velocities into an equation for velocity.
	Secondly, one could transform the controllers into possibly entirely different controllers (not just less permissive) so that they still preserve optimality (or yield $\varepsilon$-optimality) but are smaller or simpler.
	Here, a closer interaction loop with the model checkers might lead to efficient heuristics.
	\clearpage

	\bibliographystyle{splncs04}
	\bibliography{ref}
	
	\ifarxivelse{
		\clearpage
		\appendix
		\input{9_appendix}
	}
	{\clearpage}
\vfill

{\small\medskip\noindent\textbf{Open Access} This chapter is licensed under the terms of the Creative Commons\break Attribution 4.0 International License(\url{https://creativecommons.org/licenses/by/4.0/}), which permits use, sharing, adaptation, distribution and reproduction in any medium or format, as long as you give appropriate credit to the original author(s) and the source, provide a link to the Creative Commons license and indicate if changes were made.}
		
		{\small \spaceskip .28em plus .1em minus .1em The images or other third party material in this chapter are included in the chapter's Creative Commons license, unless indicated otherwise in a credit line to the material.~If material is not included in the chapter's Creative Commons license and your intended\break use is not permitted by statutory regulation or exceeds the permitted use, you will need to obtain permission directly from the copyright holder.}
		
		\medskip\noindent\includegraphics{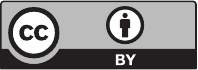}
	
\end{document}

%% file: macros.tex
\definecolor{redsalsa}{HTML}{ED474A}
\definecolor{blueyonder}{HTML}{197BBD}

\newcommand{\qee}{\hfill$\triangle$}
\DeclarePairedDelimiter\abs{\lvert}{\rvert}

\newcommand{\olddtcontrol}{\texttt{dtControl 1.0}}
\newcommand{\dtcontrol}{\texttt{dtControl 2.0}}
\newcommand{\justdtcontrol}{\texttt{dtControl}}
\newcommand{\tool}{\justdtcontrol}

\newcommand{\uppaal}{\texttt{Uppaal Stratego}}
\newcommand{\scots}{\texttt{SCOTS}}
\newcommand{\storm}{\texttt{STORM}}
\newcommand{\prism}{\texttt{PRISM}}

\newcommand{\detfun}{\delta}
\newcommand{\image}{c\mathbb{D}}
\newcommand{\im}{\image(\delta(D))}
\newcommand{\imbar}{\image(\bar{\delta}(D))}
\newcommand{\mle}{\textsc{MLE}}

\newcommand{\dec}{\mathit{dec}}
\newcommand{\neu}{\mathit{neu}}
\newcommand{\acc}{\mathit{acc}}

\newcommand{\bigcell}[2]{\begin{tabular}{@{}#1@{}}#2\end{tabular}}

%% file: 1_intro.tex
\section{Introduction}

A \emph{controller} (also known as strategy, policy or scheduler) of a system assigns to each state of the system a set of actions that should be taken in order to achieve a certain goal. For example, one may want to satisfy a given specification of a robot's behaviour or exhibit a concurrency bug appearing only in some interleaving. 
It is desirable that the controllers possess several additional properties, besides achieving the goal, in order to be usable in practice.
Firstly, controllers should be \emph{explainable}.
Only then can they be understood, trusted and implemented by the engineers, certified by the authorities, or used in the debugging process \cite{counterexample-learning}.
Secondly, they should be \emph{small} in size and efficient to run.
Only then they can be deployed on embedded devices with limited memory of a few kilobytes, while the automatically synthesized ones are orders of magnitude larger~\cite{DBLP:conf/adhs/ZapreevVM18}.
Thirdly, whenever the primary goal, e.g.\ functional correctness, is accompanied by a secondary criterion, e.g.\ energy efficiency, they should be \emph{performant} with respect to this criterion.

Automatic controller synthesis is able to provide controllers for a given goal in various domains, such as probabilistic systems \cite{prism,storm}, hybrid systems~\cite{scots,uppaal,pessoa,spaceex} or reactive systems \cite{strix}.
In some cases, even the performance can be reflected \cite{uppaal}.
However, despite recent interest in explainability in connection to AI-based controllers~\cite{XAI} and despite typically small memories of embedded devices, automatic techniques for controller synthesis mostly fall short of producing small explainable results.
A typical outcome is a controller in the form of a look-up table, listing the actions for each possible state, or a binary decision diagram (BDD)~\cite{bdd} representation thereof.
While the latter reduces the size to some extent, none of the two representations is explainable: the former due to its size, the latter due to the bit-level representation with all high-level structure lost.
Instead, learning representations in the form of decision trees (DT)~ \cite{mitchellML} has been recently explored to this end~\cite{sos-qest19,viktor-qest19}.
DTs turn out to be usually smaller than BDD but do not drown to the bit level {and} are generally well known for their interpretability and explainability due to their simple structure.
However, despite showing significant potential, the state-of-the-art tool dtControl~\cite{dtcontrol} uses predicates without natural interpretation, and moreover, the best size reductions are achieved using \emph{determinization}, i.e. making the controller less permissive, which negatively affects performance~\cite{sos-qest19}.

\begin{example}[Motivating example]\label{ex:motivating}
	Consider the cruise control model of \cite{larsen-cruise}, where we want to control the speed of our car so that it never crashes into the car in front while, as a secondary performance objective, keeping the distance between the two cars small.
	
	A safe controller for the this model as returned by \uppaal, is a lookup table of size 418 MB with 300,000 lines.
	The respective BDD has 1,448 nodes with all information bit-blasted. 
	Using adaptations of standard DT-construction algorithms, as implemented in \justdtcontrol, we can get a DT with 987 nodes, which is still too large to be explained. 
	Using determinization techniques, the controller can be compressed to 3 nodes! 
	However, then the DT allows only to decelerate until the minimum velocity. 
	This is safe, as we cannot crash into the car in front, but it does not even attempt at getting close to the front car, and thus has a very bad performance. 
	
	One can find a strategy with optimal performance, retaining the maximal permissiveness, not determinizing at all, %
	which can be represented by a DT with 11 nodes. 
	A picture of this DT as well as reasoning how to derive the predicates from the kinematic equations is in \ifarxivelse{Appendix~\ref{app:DK}}{the extended version of this paper~\cite[Appendix~A]{techreport}}.  
	
	However, exactly because the predicates are based on the \emph{domain knowledge}, namely the kinematic equations, they take the form of \emph{algebraic predicates} and not simply linear predicates, which are the only ones in \justdtcontrol\ and commonly in the machine-learning literature on DTs. \qee
\end{example}

This motivating example shows that using domain knowledge and algebraic predicates, available now in \dtcontrol, one can get smaller representation than when using existing heuristics.
Further, it improves the performance of the DT, and it is easily explainable, as it is based on domain knowledge.
In fact, the discussed controller is so explainable that it allowed us to find a bug in the original model.
In general, using \dtcontrol\ a domain expert can try to compress the controller, thus gain more insight and validate that it is correct. 
Another example of this has been reported from the use of \justdtcontrol\ in the manufacturing domain \cite{jonis-private-communication}.

While automatic synthesis of good predicates from the domain knowledge may seem as distant as automatic synthesis of program invariants or automatic theorem provers, we adopt the philosophy of those domains and offer \emph{semi-automatic techniques}.

Additionally, if not performance but only safety of a controller is relevant, we can still benefit from determinization without drawbacks. 
To this end, we also provide a new determinization procedure that generalizes the extremely successful MaxFreq technique of~\cite{dtcontrol} and is as good or better on all our examples.

To incorporate the changes just discussed, namely algebraic predicates, semi-automatic approach, and better determinization, we have also reworked the tool and its interfaces. To begin with, the software architecture of \dtcontrol\ is now very modular and allows for easy further modifications, as well as adding support for new synthesis tools.
In fact, we have already added parsers for the tools \storm~\cite{storm} and \prism~\cite{prism}, and thus we support probabilistic models as well. Since these models also contain categorical (or enumeration-type) variables, e.g. protocol states, we have also added support for categorical predicates. %
Furthermore, we added a graphical user interface that not only is easier to use than the command-line interface, but also allows to inspect the DT, modify and retrain parts of it, and simulate runs of the model under its control, further increasing the possibilities to explain the DT and validate the controller.

Summing up, the main improvements of \dtcontrol\ over the previous version \cite{dtcontrol} are the following:
\begin{itemize}
\item Support of algebraic predicates and categorical predicates
\item Semi-automatic interface and GUI with several interactive modes
\item New determinization procedure
\item Interfaces for model checkers PRISM and Storm and experimental evidence of improvements on probabilistic models compared to BDD
\end{itemize}

The paper is structured as follows.
After recalling necessary background in Section~\ref{sec:prelims}, we give an overview of the improvements over the previous version of the tool from the global perspective in Section~\ref{sec:tool}.
We detail on the algorithmic contribution in Sections \ref{sec:domain} (predicate domains), \ref{sec:selection} (predicate selection) and \ref{sec:determinize} (determinization).
Section \ref{sec:exp} provides experimental evaluation and Section~\ref{sec:conc} concludes.

\subsubsection{Related work.} 
DTs have been suggested for representing controllers of and counterexamples in probabilistic systems in \cite{counterexample-learning}, however, the authors only discuss approximate representations. The ideas have been extended to other setting, such as reactive synthesis \cite{viktor-reactivesynth} and hybrid systems \cite{sos-qest19}. More general linear predicates have been considered in leaves of the trees in \cite{viktor-qest19}. \dtcontrol\ contains the DT induction algorithms from \cite{sos-qest19,viktor-qest19}.
The differences to the previous version of the tool \justdtcontrol\ \cite{dtcontrol} are summarized above and schematically depicted in Figure~\ref{fig:tool}.

Besides, DTs have been used to represent and learn strategies for safety objectives in \cite{neider-fmcad19} and to learn program invariants in \cite{neider-dt-invariants}.
Further, DTs were used for representing the strategies during the model checking process, namely in strategy iteration ~\cite{DBLP:conf/ijcai/BoutilierDG95} or in simulation-based algorithms \cite{PnH}.
Representing controllers exactly using a structure similar to DT (mistakenly claimed to be an algebraic decision diagram) was first suggested by \cite{girard2013lowcomplexity}, however, no automatic construction algorithm was provided.

The idea of non-linear predicates has been explored in \cite{DBLP:conf/icml/IttnerS96}. In that work, however, it is not based on domain knowledge, but rather on projecting the state-space to higher dimensions.

BDDs~\cite{bdd} have been commonly used to represent strategies in planning \cite{DBLP:conf/aaai/CimattiRT98}, symbolic model checking \cite{prism} as well as to represent hybrid system controllers \cite{scots,pessoa}.
While BDD \cite{bdd} operate only on Boolean variables, they have the advantage of being diagrams and not trees.
Moreover, they correspond to Boolean functions that can be implemented on hardware easily.
\cite{DellaPenna2009} proposes an automatic compression technique for numerical controllers using BDDs.
Similar to our work, \cite{DBLP:conf/adhs/ZapreevVM18} considers the problem of obtaining concise BDD representation of controllers and presents a technique to obtain smaller BDDs via determinization.
However, BDDs are difficult to explain due to variables being bit-blasted and their size is very sensitive to the chosen variable ordering.
An extension of BDDs, algebraic or multi-terminal decision diagrams (ADD/MTBDD) \cite{ADD,MTBDD}, have been used in reinforcement learning for strategy synthesis~\cite{DBLP:conf/uai/HoeySHB99,DBLP:conf/nips/St-AubinHB00}. ADDs extend BDDs with the possibility to have multiple values in the terminal nodes, but the predicates still work only on boolean variables, retaining the disadvantages of BDDs.

%

\iffalse
Also look at Pranav's thesis RW.

\subsubsection{RW ideas that aggregated over the past few months}

CAV 5B 18:15 Controller Synthesis [Very good talk, lots of related work, might want to look at it for TACAS paper]

\url{http://proceedings.mlr.press/v80/verma18a/verma18a.pdf} (was mentioned in the keynote on NN at CAV, but had sth to do with DTs, I think (see Pranav Telegram); couldn't verify this when skimming paper)

Found some RW and described in DT-strat log 11.08. Contains Boutilier and Neider.

Shih, Choi \& Darwiche IJCAI 18 (heard about it at GandALF invited talk); explaining OBDDs algorithmically. Use it as motivation for our story "We want to explain, and we can't do automatically yet" (might want to check Darwiche's youtube channel and his newer work ("On the reasons behind decision"; sec 8.3 Computing Queries) and maybe find the Vardi longer tutorial); extend his work to DTs or use it for our purpose? Talk also mentioned robustness of NNs, maybe go into that when motivating explainability; also the discussion from CAV

PODS 2020 paper "Three modern roles for logic in AI", third part and UCLA automated reasoning group youtube channel

Think about how to handle this paper that does determinization for BDDs \url{https://reader.elsevier.com/reader/sd/pii/S2405896318311145?token=855EC1320A3293ACA69181A79171FED8A358BBC1CA55893F7C07C63E22603B6C4925D75A11B3EFEE923F4BD567D1642F}

\fi

%

%% file: 2_preliminaries.tex
\section{Decision tree learning for controller representation}\label{sec:prelims}

In this section, we briefly describe how controllers can be represented as decision trees as in~\cite{dtcontrol}. %
We give an exemplified overview of the method, pinpointing the role of our algorithmic contributions.

A (non-deterministic, also called permissive) \emph{controller} is a map $C: S \mapsto 2^A$ from states to non-empty sets of actions.
This notion of a controller is fairly general; the only requirement is that it has to be memoryless and non-randomized. 
These kind of controllers are optimal for many tasks such as expected (discounted) reward, reachability or parity objectives.
Moreover, even finite-memory controllers can be written in this form by considering the product of the state space with the finite memory as the domain, for example, like in LTL model checking.

\emph{Decision trees} (DT), e.g. \cite{mitchellML}, are trees where every leaf node is labelled with a non-empty set of actions and every inner node is labelled with a \emph{predicate} $\rho: S \mapsto \{\mathit{true},\mathit{false}\}$.

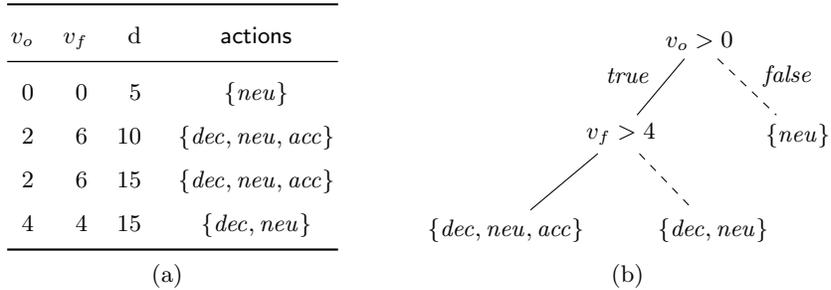
\begin{figure}
	\centering
	\newsavebox{\tempbox}
	\savebox{\tempbox}{
		{
		\renewcommand{\arraystretch}{1.5}
		\begin{tabular}[b]{r@{\hspace{1.2em}}r@{\hspace{1.2em}}r@{\hspace{1.5em}}c} \toprule
			$v_o$ & $v_f$ & d & \textsf{actions} \\ \midrule
			0 & 0 & 5 & $\{\neu\}$ \\
			2 & 6 & 10 & $\{\dec, \neu, \acc\}$ \\
			2 & 6 & 15 & $\{\dec, \neu, \acc\}$ \\
			4 & 4 & 15 & $\{\dec, \neu\}$ \\ \bottomrule
		\end{tabular}
		}
	}%
	\subfloat[]{\label{fig:ex-table}
		\usebox{\tempbox}
		
	}\qquad
	\subfloat[]{\label{fig:ex-tree}
		\begin{tikzpicture}[x=1cm,y=0.75cm]

			\node (A) at (0, 0) {$v_o > 0$};
			\node[below right=1 and 0.2 of A] (B) {$\{\neu\}$};
			\node[below left=1 and -0.1 of A] (C) {$v_f > 4$};
			\node[below right=1 and -0.2 of C] (D) {$\{\dec, \neu\}$};
			\node[below left=1 and -0.2 of C] (E) {$\{\dec, \neu, \acc\}$};
			
			\draw[-,dashed] (A) -- node[pos=0.3,xshift=20] {\small\emph{false}} (B);
			\draw[-] (A) -- node[pos=0.3,xshift=-16]  {\small\emph{true}} (C);
			\draw[-,dashed] (C) -- (D);
			\draw[-] (C) -- (E);

		\end{tikzpicture}
	}%
	\caption{An example controller based on the cruise-control model in the form of a lookup table (left), and the corresponding decision tree (right). 
	}
	\label{fig:example}
\end{figure}

\begin{example}[Decision tree representation]\label{ex:controller}
	As an example, consider the controller given in %
	 Figure~\ref{fig:ex-table}. It is a subset of the real cruise-control case study from the motivating Example \ref{ex:motivating}.
	A state is a 3-tuple of the variables $v_o$, $v_f$ and $d$, which denote the velocity of our car, the front car and the distance between the cars respectively.
	In each state, our car may be allowed to perform a subset of the following set of actions: decelerate ($\dec$), stay in neutral ($\neu$) or accelerate ($\acc$).
	A DT representing this lookup table is depicted in %
	 Figure~\ref{fig:ex-tree}.
	
	Given a state, for example $v_o = v_f = 4, d = 10$, the DT is evaluated as follows:
	We start at the root and, since it is an inner node, we evaluate its predicate $v_o > 0$. As this is true, we follow the true branch and reach the inner node labelled with the predicate $v_f > 4$. This is false, so we follow the false branch and reach the leaf node labelled $\{\dec,\neu\}$. Hence, we know that all three possibilities of decelerating, staying neutral and accelerating are allowed by the controller. \qee
\end{example}

To construct a DT representation of a given controller, the following recursive algorithm may be used. Note that it is heuristic since constructing an optimal binary decision tree is an NP-complete problem \cite{DBLP:journals/ipl/HyafilR76}.
\begin{itemize}
	\item[Base case:] 
	If all states in the the controller agree on their set of actions $B$ (i.e. for all states $s$ we have $C(s) = B$), return a leaf node with label $B$.
	 
	\item[Recursive case:] Otherwise, we split the controller. 
	For this, we select a predicate $\rho$ and construct an inner node with label $\rho$.
	Then we partition the controller by evaluating the predicate on the state space, and recursively construct one DT for the sub-controller on states $\{s \in S \mid \rho(s)\}$ where the predicate is true, and one for the sub-controller where it is false.
	These controllers are the children of the inner node with label $\rho$ and we proceed recursively. 
\end{itemize}

For selecting the predicate, we consider two hyper-parameters: 
The \emph{domain} of the predicates (see Section~\ref{sec:domain}) and the way to \emph{select} predicates (see Section~\ref{sec:selection}). The selection is typically performed by selecting the predicate with the lowest \emph{impurity}; this is a measure for how homogenous (or ``pure'') the controller is after the split, in other words the degree to which all the states agree on their actions.

We also consider a third hyper-parameter of the algorithm, namely \emph{determinization} by \emph{safe early stopping} (see Section \ref{sec:determinize}). This modifies the base case as follows: if all states in the controller agree on at least one action $a$ (i.e. for all states $s$ we have $a \in C(s)$), then we return a leaf node with label $\{a\}$. 
This variant of early stopping ensures that, even though the controller is not represented exactly, still for every state a safe action is allowed.

Hence, if the original controller satisfies some property, e.g. that a safe set of states is never left, the DT construction algorithm ensures that this property is retained.
This is because our algorithm represents the strategy exactly (or a safe subset, in case of determinization) and does not generalize as DTs typically do in machine learning.
DTs are suitable for both tasks, as both rely on the strength of DTs exploiting underlying structure.

\begin{remark} \label{rem:determinization-restriction}
	Note that for some types of objectives such as reachability, determinization of permissive strategies might lead to a violation of the original guarantees. 
	For example, consider a strategy that allows both a self-looping and a non-self-looping action at a particular state. 
	If the determinizer decides to restrict to the self-looping action, the reachability property may be violated in the determinized strategy. 
	However, this problem can be addressed when synthesizing the strategy by ensuring that every action makes progress towards the target.
\end{remark}

%% file: 3_tool.tex
\section{Tool}\label{sec:tool}

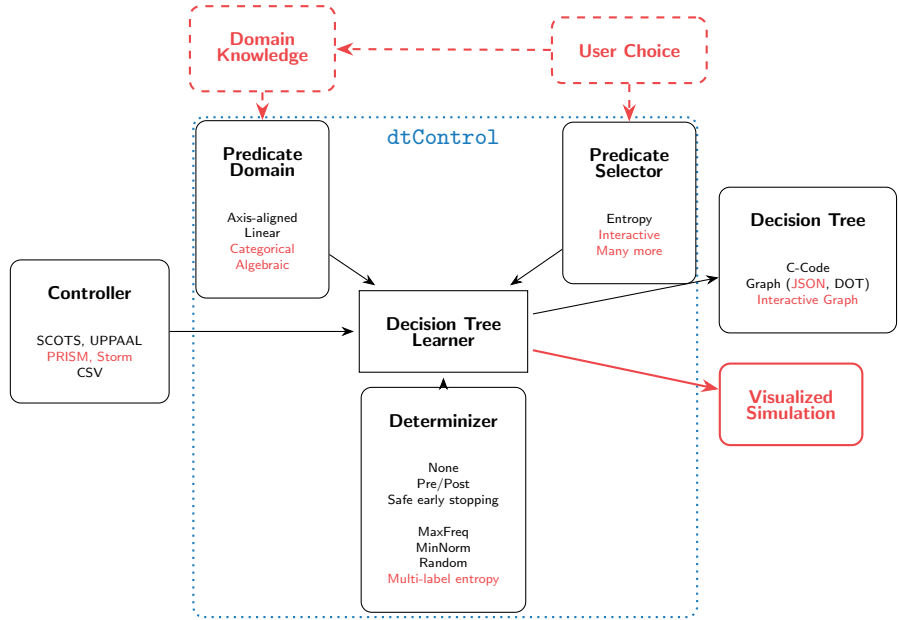
\begin{figure}[t]
	\centering
	\begin{tikzpicture}
		[scale=0.95,
		font=\sffamily\tiny,
		inner/.style={draw, rounded corners, inner sep=10},
		new/.style={inner,thick,redsalsa}
		]
		\draw[thick, rounded corners, dotted, color=blueyonder] (-3.5, 3) rectangle (3.55,-4);
		\node[above, draw=none, color=blueyonder] at (0,2.5) {\footnotesize\texttt{dtControl}};
		
		\node[inner,align=center,rounded corners=0,outer sep=2] (learning) {\scriptsize\textbf{Decision Tree}\\\scriptsize\textbf{Learner}};
		
		\node[inner,align=center] (preddom) at ($(learning.west)!0.57!(-3.5, 3)$) {\scriptsize\textbf{Predicate}\\\scriptsize\textbf{Domain}\\\vspace{1em}\\Axis-aligned\\Linear\\\textcolor{redsalsa}{Categorical}\\\textcolor{redsalsa}{Algebraic}};
		
		\node[inner,align=center] (contr) at (-4.95, 0) {\scriptsize\textbf{Controller}\\\vspace{1em}\\SCOTS, UPPAAL\\\textcolor{redsalsa}{PRISM, Storm}\\CSV};
		
		\node[inner,align=center] (predsel) at ($(learning.east)!0.6!(3.5, 3)$) {\scriptsize\textbf{Predicate}\\\scriptsize\textbf{Selector}\\\vspace{1em}\\Entropy\\\textcolor{redsalsa}{Interactive}\\\textcolor{redsalsa}{Many more}};
		
		\node[inner,align=center] (det) at ($(learning.south)!0.6!(0,-3.5)$) {\scriptsize\textbf{Determinizer}\\\vspace{1em}\\
			None\\Pre/Post\\Safe early stopping\\ \\MaxFreq\\MinNorm\\Random\\\textcolor{redsalsa}{Multi-label entropy}};
		
		\node[inner,align=center] (dtout) at (5.1, 1) {\scriptsize\textbf{Decision Tree}\\\vspace{1em}\\C-Code\\Graph (\textcolor{redsalsa}{JSON}, DOT)\\\textcolor{redsalsa}{Interactive Graph}};
		
		\node[new,align=center,below=1.5cm of dtout.south,anchor=south east,xshift=21] (viz) {\scriptsize\textbf{Visualized}\\\scriptsize\textbf{Simulation}};
		
		\node[new,align=center,dashed] (userchoice) at ($(predsel.north)+(0,1)$) {\scriptsize\textbf{User Choice}};
		
		\node[new,align=center,dashed] (domknow) at ($(preddom.north)+(0,1)$) {\scriptsize\textbf{Domain}\\\scriptsize\textbf{Knowledge}};
		
		\draw[-Stealth] (preddom) -- (learning);
		\draw[-Stealth] (predsel) -- (learning);
		\draw[-Stealth] (det) -- (learning);
		\draw[-Stealth] (contr) -- (learning);
		\draw[-Stealth] (learning) -- (dtout);
		\draw[-Stealth,new] (learning) -- (viz);
		\draw[-Stealth,new,dashed] (userchoice) -- (predsel);
		\draw[-Stealth,new,dashed] (userchoice) -- (domknow);
		\draw[-Stealth,new,dashed] (domknow) -- (preddom);
		
	\end{tikzpicture}
	\caption{An overview of the components of \dtcontrol, thereby showing software architecture and workflow. Contributions of this paper are highlighted in red.}
	\label{fig:tool}
\end{figure}

\dtcontrol\ is an easy-to-use open-source tool for representing memoryless symbolic controllers as more compact and more interpretable DTs, while retaining safety guarantees of the original controllers.
Our website \url{dtcontrol.model.in.tum.de} offers hyperlinks to the easy-to-install \texttt{pip} package\footnote{\texttt{pip} is a standard package-management system used to install and manage software packages written in Python.}, the documentation and the source code. 
Additionally, the artifact that has passed the TACAS 21 artifact evaluation is available here \cite{artifact}.

The schema in Figure \ref{fig:tool} illustrates the workflow of using \tool, highlighting new features in red.
Considering \tool\ as a black box, it shows that given a controller, it returns a DT representing the controller and also offers the possibility to simulate a run of the system under the control of the DT, visualizing the decisions made.
The controller can be input in various formats, including the newly supported strategy representations of the well-known probabilistic model checkers \prism~\cite{prism} and \storm~\cite{storm}.
The DT is output in several machine readable formats, and as C-code that can be directly used for executing the controller on embedded devices. Note that this C-code consists only of nested if-else-statements.
The new graphical user interface also offers the possibility to inspect the graph in an interactive web user interface, which even allows to edit the DT. This means that parts of the DT can be retrained with a different set of hyper-parameters and directly replaced. This way, one can for example first train a determinized DT and then retrain important parts of it to be more permissive and hence more performant for a secondary criterion. 
Figure \ref{fig:GUI} shows a screenshot of the newly integrated graphical user interface.

\begin{figure}[t]
	\centering
		\makebox[\textwidth]{\includegraphics[width=\textwidth]{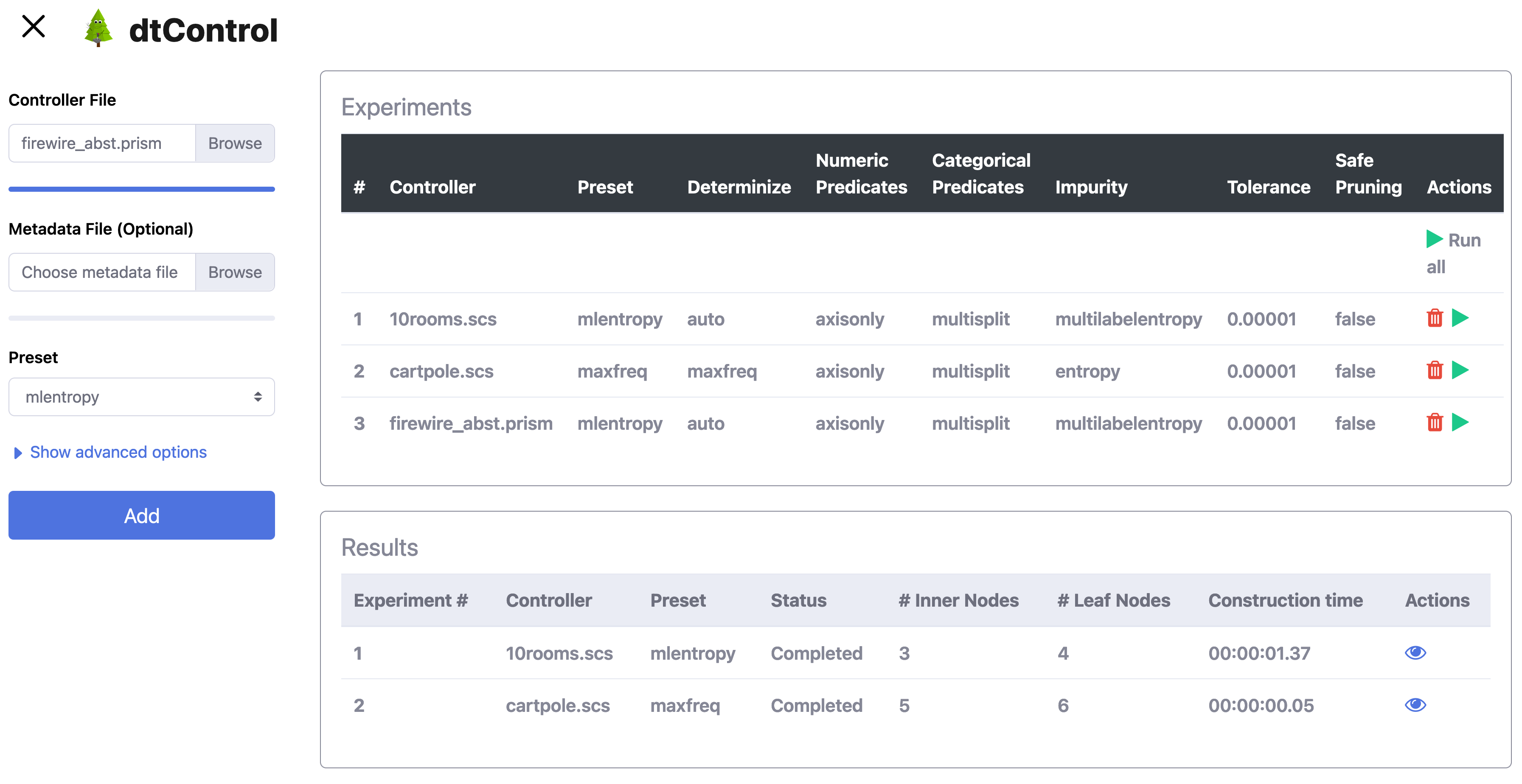}}
	\caption{Screenshot of the new web-based graphical user interface. It offers a sidebar for easy selection of the controller file and hyper-parameters, an experiments table where benchmarks can be queued, and a results table in which some statistics of the run are provided. Moreover, users can click on the `eye' icon in the results table to inspect the built decision tree.}
	\label{fig:GUI}
\end{figure}

Looking at the inner workings of \tool, we see the three important hyper-parameters that were already introduced in Section \ref{sec:prelims}: predicate domain, predicate selector, and determinizer.
For each of these, \tool\ offers various choices, some of which were newly added for version 2.0.
Most prominently, the user now has the possibility to directly influence both the predicate domain and the predicate selector, by providing domain knowledge and thus also additional predicates, or by directly using the interactive predicate selection.
More details on the predicate domain and how domain knowledge is specified can be found in Section \ref{sec:domain}. 
The different ways to select predicates, especially the new interactive mode, are the topic of Section \ref{sec:selection}. 
Our new insights into determinization are described in Section~\ref{sec:determinize}.
To support the user in finding a good set of hyper-parameters, \tool\ also offers extensive benchmarking functionality, allowing to specify multiple variants and reporting several statistics.

\subsubsection*{Technical notes.}
\dtcontrol\ is written in Python 3 following an architecture closely resembling the schema in Figure \ref{fig:tool}.
The modularity, along with our technical documentation, allows users to easily extend the tool.
For example, supporting another input format is only a matter of adding a parser.

\dtcontrol\ works with Python version 3.7.9 or higher. The core of the tool which runs the learning algorithms requires \texttt{numpy}~\cite{numpy}, \texttt{pandas}~\cite{pandas} and \texttt{scikit-learn}~\cite{scikit-learn} and optionally the library for the heuristic \texttt{OC1}~\cite{OC1}.
The algebraic predicates rely on \texttt{SymPy}~\cite{sympy} and \texttt{SciPy}~\cite{scipy}.
The web user interface is powered by Flask~\cite{flask} and D3.js~\cite{d3}.

%% file: 4_theory.tex
\section{Predicate domain}\label{sec:domain}

The domain of the predicates that we allow in the inner nodes of the DT is of key importance. As we saw in the motivating Example \ref{ex:motivating}, allowing for more expressive predicates can dramatically reduce the size of the DT.

We assume that our state space is structured, i.e. it is a Cartesian product of the domain of the variables ($S = S_1 \times \dots \times S_n$). 
We use $s_i$ to refer to the $i$-th state-variable of a state $s \in S$.
In Example \ref{ex:controller}, the three state-variables are the velocity of our car, the velocity of the front car, and the distance. 

We first give an overview of the predicate domains \dtcontrol\ supports, before discussing the details of the new ones.

\emph{Axis-aligned predicates}~\cite{mitchellML} have the form $s_i \leq c$, where $c$ is a rational constant. 
This is the easiest form of predicates, and they have the advantage that there are only finitely many, as the domain of every state-variable is bounded.
However, they are also least expressive.

\emph{Linear predicates} (also known as oblique~\cite{OC1}) have the form $\sum_i s_i \cdot a_i \leq c$, where $a_i$ are rational coefficients and $c$ is a rational constant. 
They have the advantage that they are able to combine several state-variables which can lead to saving linearly many splits, cf.~\cite[Fig. 5.2]{MJ}. 
The disadvantage of these predicates is that there are infinitely many choices of coefficients, which is why  heuristics were introduced to determine a good set of predicates to try out~\cite{OC1,dtcontrol}.
However, heuristically determined coefficients and combinations of variables can impede explainability.

\emph{Algebraic predicates} have the form $f(s) \leq c$, where $f$ is any mathematical function over the state-variables and $c$ is a rational constant. It can use elementary functions such as exponentiation, $\log$, or even trigonometric functions. Example \ref{ex:motivating} illustrated how this can reduce the size and improve explainability. More discussion of these predicates follows in Section~\ref{sec:alg}.

\emph{Categorical predicates} are special predicates for categorical (enumeration-type) state-variables such as colour or protocol state, and they are discussed in Section~\ref{sec:cat}.

\subsection{Categorical predicates}\label{sec:cat}

Categorical state-variables do not have a numeric domain, but instead are unordered and qualitative. 
They commonly occur in the models coming from the tools \prism\ and \storm.

\begin{example}
	Let one state-variable be `colour' with the domain $\{\text{red},\text{blue},\text{green}\}$.
	A simple approach is to assign numbers to every value, e.g. $\text{red}=0, \text{blue}=1, \text{green}=2$, and treat this variable as numeric.
	However, a resulting predicate such as $\text{colour} \leq 2$ is hardly explainable and additionally depends on the assignment of numbers. For example, it would not be possible to single out $\text{colour} \in \{\text{red},\text{green}\}$ using a single predicate, given the aforementioned numeric assignment. 
	Using linear predicates, for example adding half of the colour to some other state-variable, is even more confusing and dependent on the numeric assignment. \qee
\end{example}

\begin{figure}[bp]
	\centering
	\subfloat[]{
		\begin{tikzpicture}[x=1cm,y=1cm]

		\node (A) at (0, 0) {\textsf{color}};
		\node[below right=1 and 0.5 of A] (B) {$\{c\}$};
		\node[below left=1 and 0.5 of A] (C) {$\{a\}$};
		\node[below=1 of A] (D) {$\{b\}$};
		
		\draw[-] (A) -- node[pos=0.5,xshift=6] {b} (B);
		\draw[-] (A) -- node[pos=0.5,xshift=-6]  {r} (C);
		\draw[-] (A) -- node[pos=0.5,xshift=6]  {g} (D);			
		\end{tikzpicture}
	}\qquad\qquad
	\subfloat[]{
		\begin{tikzpicture}[x=1cm,y=1cm]

		\node (A) at (0, 0) {\textsf{color}};
		\node[below right=1 and 0.5 of A] (B) {$\{b\}$};
		\node[below left=1 and 0.5 of A] (C) {$\{a\}$};
		
		\draw[-] (A) -- node[pos=0.5,xshift=6] {b} (B);
		\draw[-] (A) -- node[pos=0.5,xshift=-10]  {r, g} (C);			
		\end{tikzpicture}
	}%
	\caption{Two examples of a categorical split. On the left, all possible values of the state-variable \textsf{colour} lead to a different child in a non-binary split. On the right, red and green lead to the same child, which is a result of grouping similar values together.}
	\label{fig:catEx}
\end{figure}
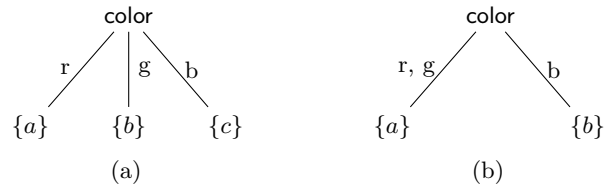

Instead of treating the categorical variables using their numeric encodings, \dtcontrol\ supports specialized algorithms from literature, see e.g.~\cite{Qui86,Qui93}.
They work by labelling an inner node with a categorical variable and performing a (possibly non-binary) split according to the value of the categorical variable. 
The node can have at most one child for every possible value of the categorical variable, but it can also group together similarly behaving values, see Figure \ref{fig:catEx} for an example. 
For the grouping, \dtcontrol\ uses the greedy algorithm from~\cite[Chapter 7]{Qui93} called attribute-value grouping.
It proceeds by first considering to have a branch for every single possible value of the categorical variable, and then merging branches as long as it improves the predicate; see \ifarxivelse{Appendix~\ref{app:avg-alg}}{\cite[Appendix~C]{techreport}} for the full pseudocode of the algorithm.

In our experiments we found that the grouping algorithm sometimes did not merge branches in cases where it would actually have made the DT smaller or more explainable. This is because the resulting impurity, the goodness of a predicate, could be marginally worse due to floating-point inaccuracies.
Thus, we introduce \emph{tolerance}, a bias parameter in favour of larger value groups. When checking whether to merge branches, we do not require the impurity to improve, but we allow it to become worse up to our tolerance.
Setting tolerance to 0 corresponds exactly to the algorithm from~\cite{Qui93}, while setting tolerance to $\infty$ results in merging branches until only two remain, thus producing binary predicates.

To allow \dtcontrol\ to use categorical predicates, the user has to provide a metadata file, which tells the tool which variables are categorical and which are numeric; see \ifarxivelse{Appendix~\ref{app:metadatafile}}{\cite[Appendix B.1]{techreport}} for an example.

\subsection{Algebraic predicates}\label{sec:alg}
It is impossible to try out every mathematical expression over the state-variables, and it would also not necessarily result in an explainable DT.
Instead, we allow the user to enter \emph{domain knowledge} to suggest templates of predicates that \dtcontrol\ should try. See \ifarxivelse{Appendix~\ref{app:DK-file}}{\cite[Appendix~B.2]{techreport}} for a discussion of the format in which domain knowledge can be entered.

Providing the basic equations that govern the model behaviour can already help in finding a good predicate, and is easy to do for a domain expert.
Additionally, \dtcontrol\ offers several possibilities to further exploit the provided domain knowledge:

Firstly, the given predicates need not be exact, but may contain coefficients. These coefficients can be both completely arbitrary or may come from a finite set suggested by the user. For coefficients with finite domain, \dtcontrol\ tries all possibilities; for arbitrary coefficients, it uses curve fitting to find a good value.
For example, the user can specify a predicate such as $d + (v_o - v_f) \cdot c_0 > c_1$ with $c_0$ being an arbitrary rational number and $c_1 \in \{0,5,10\}$.

Secondly, the interactive predicate selection (see Section ~\ref{sec:selection}) allows the user to try out various predicates at once and observe their respective impurity in the current node.
The user can then choose among them as well as iteratively suggest further predicates, inspired by those where the most promising results were observed.

Thirdly, the decisions given by a DT can be visualized in the simulator, possibly leading to better understanding the controller.
Upon gaining any further insight, the user can directly edit any subtree of the result, possibly utilizing the interactive predicate selection again.

\section{Predicate selection}\label{sec:selection}

The tool offers a range of options to affect the selection of the most appropriate predicate from a given domain.

\paragraph{Impurity measures:}
As mentioned in Section \ref{sec:prelims}, the predicate selection is typically based on the lowest \emph{impurity} induced. 
The most commonly used impurity measure (and the only one the first version of \tool\ supported) is Shannon's entropy~\cite{shannon}.
In \dtcontrol, a number of other impurity measures from the literature~\cite{Qui86,Breiman1984,Heath1993,OC1,viktor-qest19} are available. 
However, our results indicate that entropy typically performs the best, and therefore it is used as the default option unless the user specifies otherwise.
Due to lack of space, we delegate the details and experimental comparison between the impurity measures to \ifarxivelse{Appendix~\ref{app:impurities}}{\cite[Appendix~D]{techreport}}.

\paragraph{Priorities:}
\dtcontrol\ also has the new functionality to assign \emph{priorities} to the predicate generating algorithms.
Priorities are rational numbers between 0 and 1. 
The impurity of every predicate is divided by the priority of the algorithm that generated it.
For example, a user can use axis-aligned splits with priority 1 and a linear heuristic with priority $\nicefrac 1 2$. Then the more complicated linear predicate is only chosen if it is at least twice as good (in terms of impurity) as the easier-to-understand axis-aligned split. 
A predicate with priority 0 is only considered after all predicates with non-zero priority have failed to split the data.
This allows the user to give just a few predicates from domain knowledge, which are then strictly preferred to the automatically generated ones, but which need not suffice to construct a complete DT for the controller.

\paragraph{Interactive predicate selection:}
\dtcontrol\ offers the user the possibility to manually select the predicate in every split. 
This way, the user can prefer predicates that are explainable over those that optimize the impurity.

The screenshot of the interactive interface in~\ifarxivelse{Appendix \ref{app:interScreenshot}}{\cite[Appendix~F]{techreport}} shows the information that \dtcontrol\ provides. 
The user is given some statistics and metadata, e.g. minimum, maximum and step size of the state-variables in the current node, a few automatically generated predicates for reference and all predicates generated from domain knowledge. 
The user can specify new predicates and is immediately informed about their impurity. 
Upon selecting a predicate, the split is performed and the user continues in the next~node. 

The user can also first construct a DT using some automatic algorithm and then restart the construction from an arbitrary node using the interactive predicate selection to handcraft an optimized representation, or at any point decide that the rest of the DT should be constructed automatically.

\section{New insights about determinization}\label{sec:determinize}

In our context, \emph{determinization} denotes a procedure that, for some or all states, picks a subset of the allowed actions.
Formally, a determinization function $\detfun$ transforms a controller $C$ into a ``more determinized'' $C'$, such that for all states $s \in C$ we have $\emptyset \subsetneq C'(s) \subseteq C(s)$.
This reduces the permissiveness, but often also reduces the size.
Note that, for safety controllers, this always preserves the original guarantees of the controller. For other (non-safety) controllers, see Remark \ref{rem:determinization-restriction}.

\dtcontrol\ supports three different general approaches to determinizing a controller: pre-processing, post-processing and safe early stopping. 
Pre-processing commits to a single determinization before constructing the DT.
Post-processing prunes the DT after its construction, e.g. safe pruning in~\cite{sos-qest19}.
The basic idea of safe early stopping is already described in Section~\ref{sec:prelims}: if all states agree on at least one action, then instead of continuing to split the controller, stop early and return a leaf node with that common action.
Alternatively, to preserve more permissiveness, one can return not only a single common action, but all common actions; formally, return the maximum set $B$ such that for all states $s$ in the node $B \subseteq C(s)$.

The results of~\cite{dtcontrol} show that both pre-processing and post-processing are outperformed by an on-the-fly approach based on safe early stopping.
This is because pre-processing discards a lot of information that could have been useful in the DT construction and post-processing can only affect the bottom-most nodes of the resulting DT, but usually not those close to the root.

We now give a new view on safe early stopping approaches for determinizing a controller that allows us to generalize the techniques of~\cite{dtcontrol}, reducing the size of the resulting DTs even more.

\begin{example}
	Consider the following controller:
	$C(s_1) = \{a,b,c\}$, $C(s_2) = \{a,b,d\}$, $C(s_3) = \{x,y\}$.
	All three states map to different sets of actions, and thus an impurity measure like entropy penalizes grouping $s_1$ and $s_2$ the same as grouping $s_1$ and $s_3$.
	However, if determinization is allowed, grouping $s_1$ and $s_2$ need not be penalized at all, as these states agree on some actions, namely $a$ and $b$. 
	Grouping $s_1$ and $s_2$ into the same child node thus allows the algorithm to stop early at that point and return a leaf node with $\{a, b\}$, in contrast to grouping $s_1$ and $s_3$. \qee
\end{example}

Knowing that we want to determinize by safe early stopping affects the predicate selection process.
Intuitively, sets of states are more homogeneous the more actions they share.
We want to take this into account when calculating the impurity of predicates.
One way to do this would be to calculate the impurity of all possible determinization functions and pick the best one.
This, however, is infeasible, hence we propose the heuristic of \emph{multi-label impurity measures}. These impurity measures do not only consider the full set of allowed actions in their calculation, but instead they depend on the individual actions occurring in the set. This allows the DT construction to pick better predicates, namely those whose resulting children are more likely to be determinizable. In~\ifarxivelse{Appendix \ref{app:MLE}}{\cite[Appendix~E]{techreport}} we formally derive the multi-label variants of entropy and Gini-index. 

To conclude this section, we point out the key difference between the new approach of multi-label impurity measures and the previous idea that was introduced in~\cite{dtcontrol}.
The approach from~\cite{dtcontrol} does not evaluate the impurity of all possible determinization functions, but rather picks a smart one -- that of maximum frequency (MaxFreq) -- and evaluates according to that. 
MaxFreq determinizes in the following way: for every state, it selects from the allowed actions that action occurring most frequently throughout the whole controller. This way, many states share common actions.
This is already better than pre-processing, as it does not determinize the controller a priori, but rather considers a different determinization function at every node. However, in every node we calculate the impurity for several different predicates, and the optimal choice of determinization function depends on the predicate. Thus, choosing a single determinization function for a whole node is still too coarse, as it is fixed independent of the considered predicate. We illustrate the arising problem in the following Example~\ref{ex:mle}.

\begin{example}\label{ex:mle}
	\begin{figure}
		\centering
		\begin{tikzpicture}[x=0.9cm,y=0.9cm]
			\draw[-Stealth] (0,0) -- (4,0) node[anchor=north west] {\textsf{x}};
			\draw[-Stealth] (0,0) -- (0,4) node[anchor=south east] {\textsf{y}};
			\foreach \x in {0,1,2,3}
			\draw (\x cm,1pt) -- (\x cm,-1pt) node[anchor=north] {\x};
			\foreach \y in {0,1,2,3}
			\draw (1pt,\y cm) -- (-1pt,\y cm) node[anchor=east] {\y};

			\draw[-,blueyonder,very thick] (1.5,0) -- (1.5,4);
			\draw[-,redsalsa,very thick] (0,1.5) -- (3,1.5);	
			
			\node[circle,fill,inner sep=1pt] (x13) at (1,3) {} node[left] at (x13) {$\{a, c\}$};
			\node[circle,fill,inner sep=1pt] (x12) at (1,2) {} node[left] at (x12) {$\{a, c\}$};
			\node[circle,fill,inner sep=1pt] (x11) at (1, 1) {} node[below left] at (x11) {$\{a\}$};
			\node[circle,fill,inner sep=1pt] (x23) at (2, 3) {} node[right] at (x23) {$\{b, c\}$};
			\node[circle,fill,inner sep=1pt] (x22) at (2, 2) {} node[right] at (x22) {$\{b, c\}$};
			\node[circle,fill,inner sep=1pt] (x21) at (2, 1) {} node[below right] at (x21) {$\{b\}$};
			
		\end{tikzpicture}
		\caption{A simple example of a dataset that is split suboptimally by the MaxFreq approach from \cite{dtcontrol}, but optimally by the new multi-label entropy approach.}
		\label{fig:mle}
	\end{figure}
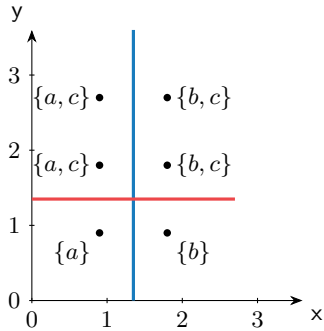
	
	Figure \ref{fig:mle} shows a simple controller with a two-dimensional state space. Every point is labeled with its set of allowed actions. 
	
	As $c$ is the most frequent action, MaxFreq determinizes the states $(1,2)$, $(1,3)$, $(2,2)$ and $(2,3)$ to action $c$. Hence the red split (predicate $y < 1.5$) is considered optimal, as it groups together all four states that map to $c$. The blue split (predicate $x<1.5$) is considered suboptimal, as then the data still looks very heterogeneous. 
	So, using MaxFreq, we need two splits for this controller; one to split of all the $c$'s and one to split the two remaining states.
	
	However, it is better to first choose a predicate and then determine a fitting determinization function. When calculating the impurity of the blue split, we can choose to determinize all states with $x=1$ to $\{a\}$ and all states with $x=2$ to $\{b\}$. Thus, in both resulting sub-controllers the impurity is 0 as all states agree on at least one action. This way, one split suffices to get a complete DT.
	Multi-label impurity measures notice when labels are shared between many (or all) states in a sub-controller, and thus they allow to prefer the optimal blue split.
	\qee
\end{example}

%% file: 5_experiments.tex
\section{Experiments}\label{sec:exp}

\paragraph{Experimental setup.}
We compare three approaches: BDDs, the first version of \tool\ from \cite{dtcontrol} and \dtcontrol.
For BDDs\footnote{Our implementation of BDDs is based on the \texttt{dd} python library \url{https://github.com/tulip-control/dd}.} the variable ordering is important, so we report the smallest of 20 BDDs that we constructed by starting with a random initial variable ordering and reordering until convergence. To determinize BDDs, we used the pre-processing approach, 10 times with the minimum norm and 10 times with MaxFreq.
For the previous version of \tool, we picked the smaller of either a DT with only axis-aligned predicates or a DT with linear predicates using the logistic regression heuristic that was typically best in~\cite{dtcontrol}. Determinization uses safe early stopping with the MaxFreq approach.
For \dtcontrol, we use the multi-label entropy based determinization and utilize the categorical predicates for the case studies from probabilistic model checking.
We ran all experiments on a server with operating system Ubuntu 19.10, a 2.2GHz Intel(R) Xeon(R) CPU E5-2630 v4 and 250 GB RAM.

\paragraph{Comparing determinization techniques on cyber-physical systems.}
Table \ref{tab:det} shows the sizes of determinized BDDs and DTs on the permissive controllers of the tools \scots\ and \uppaal\ that were already used in~\cite{dtcontrol}. We see that the new determinization approach is strictly better than the previous one, with only two DTs being of equal size, as the result of the previous method was already optimal. With the exception of the case studies helicopter and truck\_trailer where BDDs are comparable or slightly better, both approaches using DTs are orders of magnitude smaller than BDDs or an explicit representation of the state-action mapping.

\begin{table}[t]
	\caption{Controller sizes of different determinized representations of the controllers from \scots\ and \uppaal. ``States'' is the number of states in the controller, ``BDD'' the number of nodes of the smallest BDD from 20 tries, \olddtcontrol\ \cite{dtcontrol} the smallest DT the previous version of \tool\ could generate and \dtcontrol\ the smallest DT the new version can construct. ``TO'' denotes a failure to produce a result in 3 hours. The smallest numbers in each row are highlighted.}
	\label{tab:det}
	\centering{
	\ifarxivelse{\renewcommand{\arraystretch}{1.4}}{\renewcommand{\arraystretch}{1.05}}
	\begin{tabular}{@{\hspace{1.2em}}l@{\hspace{1.2em}}r@{\hspace{1.2em}}r@{\hspace{1.2em}}r@{\hspace{1.2em}}r@{\hspace{1.2em}}} \toprule
		Case study          & States   & BDD  & \olddtcontrol       & \dtcontrol           \\
		\midrule
		cartpole       & 271      & 127  & 11         & \textbf{7}     \\
		10rooms        & 26,244    & 128  & \textbf{7} & \textbf{7}     \\
		helicopter     & 280,539   & 870 & 221        & \textbf{123}   \\
		cruise-latest  & 295,615   & 1,448 & \textbf{3} & \textbf{3}     \\
		dcdc           & 593,089   & 381  & 9          & \textbf{5}     \\
		truck\_trailer & 1,386,211  & \textbf{18,186} & 42,561      & 31,499 \\
		traffic\_30m   & 16,639,662 & TO & 127        & \textbf{97}   \\
		\bottomrule
	\end{tabular}
	}
\end{table}

\paragraph{Case studies from probabilistic model checking.}

\begin{table}[t]
	\caption{Controller sizes of different representations of controllers from the quantitative verification benchmark set~\cite{qcomp}, i.e. from the tools \storm\ and \prism. ``States'' is the number of states in the controller, ``BDD'' the number of nodes of the smallest BDD of 20 tries and \dtcontrol\ the smallest DT we could construct. The smallest numbers in each row are highlighted.}
	\label{tab:mdp}
	\centering
	{
	\ifarxivelse{\renewcommand{\arraystretch}{1.4}}{\renewcommand{\arraystretch}{1.1}}
	\begin{tabular}{@{\hspace{1.2em}}l@{\hspace{1.2em}}r@{\hspace{3.8em}}r@{\hspace{1.2em}}r@{\hspace{1.2em}}} \toprule
		Case study                             & States & BDD   & \dtcontrol   \\
		\midrule
		triangle-tireworld.9              & 48     & 51    & \textbf{23}     \\
		pacman.5                          & 232    & 330   & \textbf{33}     \\
		rectangle-tireworld.11            & \textbf{241}    & 498   & 373    \\
		philosophers-mdp.3                & 344    & 295   & \textbf{181}    \\
		firewire\_abst.3.rounds           & 610    & 61   & \textbf{25}     \\
		rabin.3                           & 704    & 303   & \textbf{27}     \\
		ij.10                             & 1,013  & \textbf{436}   & 753    \\
		zeroconf.1000.4.true.correct\_max & 1,068  & 386   & \textbf{63}     \\
		blocksworld.5                     & 1,124   & 3,985  & \textbf{855}    \\
		cdrive.10                         & \textbf{1,921}   & 5,134  & 2,401   \\
		consensus.2.disagree              & 2,064   & 138   & \textbf{67}     \\
		beb.3-4.LineSeized                & 4,173   & 913   & \textbf{59}     \\
		csma.2-4.some\_before             & 7,472   & 1,059  & \textbf{103}    \\
		eajs.2.100.5.ExpUtil              & 12,627  & 1,315  & \textbf{153}    \\
		elevators.a-11-9                  & 14,742  & \textbf{6,750}  & 9,883   \\
		exploding-blocksworld.5           & 76,741  & 34,447 & \textbf{1,777}   \\
		echoring.MaxOffline1              & 104,892 & 43,165  & \textbf{1,543}   \\
		wlan\_dl.0.80.deadline            & 189,641 & 5,738  & \textbf{2,563}   \\
		pnueli-zuck.5                     & 303,427 & \textbf{50,128}  & 150,341\\
		\bottomrule
	\end{tabular}
	}
\end{table}

For Table \ref{tab:mdp}, we used case studies from the quantitative verification benchmark set~\cite{qcomp}, which includes models from the \prism\ benchmark suite~\cite{prismBench}. Note that these case studies contain unordered enumeration-type state-variables for which we utilize the new categorical predicates.
To get the controllers, we solved the case study with \storm\ and exported the resulting controller. This export already eliminates unreachable states.
The previous version of \tool\ was not able to handle these case studies, so we only compare \dtcontrol\ to BDDs. 

Table \ref{tab:mdp} shows that also for case studies from probabilistic model checking, DTs are a good way of representing controllers. The DT is the smallest representation on 13 out of 19 case studies, often reducing the size by an order of magnitude compared to BDDs or the explicit representation. On 3 case studies, BDDs are smallest, and on 2 case studies, both the DT and the BDD fail to reduce the size compared to the explicit representation. This happens if there are many different actions and thus states cannot be grouped together. A worst case example of this is a model where every state has a different action; then, a DT would have as many leaf nodes as there are states, and hence twice as many nodes in total.

\begin{remark}
Note that the controllers exported by \storm\ are deterministic, so no determinization approach can be utilized in the DT construction. 
We conjecture that if a permissive strategy was exported, \dtcontrol\ would benefit from the additional information and be able to reduce the controller size further as for the cyber-physical systems. 
\end{remark}

%% file: 9_appendix.tex
\section{Deriving algebraic predicates from domain knowledge for the cruise-control model}\label{app:DK}

The cruise-control model is governed by the kinematic equations, i.e. the new distance between the cars is computed as follows:
\[ d_{\textit{new}} = d + (v_f-v_o) \cdot t + 0.5 \cdot (a_f-a_o) \cdot t^2, \]
where $t$ is a time span in seconds, $d$, $v_f$ and $v_o$ are distance between the cars, velocity of the front car and velocity of our car as before, and $a_f$ and $a_o$ are the acceleration that the cars use during the whole time span. The model restricts these accelerations to be from the set $\{-2,0,2\}$, which corresponds to the actions deceleration, neutral or acceleration.

We want that the distance between the cars always is greater than some threshold. 
The worst case behaviour of the front car is to always decelerate, corresponding to emergency braking.
We can thus use $a_f = -2$ and the only unknown in the equations is our acceleration $a_o$. 
Now we can calculate the worst-case distance between the cars, assuming we accelerate for one time step (i.e. $a_o=2$) and then brake ($a_o=-2$) until both cars are at minimum velocity. If that distance is greater than our threshold, we know that it is safe to accelerate in the next time step.
Since the actions are ordered, in any state in which it is safe to accelerate, we can also stay neutral or decelerate (excluding the corner case of minimum velocity).
Similarly, we can split of the states that allow to stay neutral or decelerate, but not to accelerate. All remaining states only allow deceleration.

See Figure \ref{fig:SmartCruise} for the DT using algebraic predicates to represent the whole permissive controller with 11 nodes. $t_f$ is the time until the front car reaches minimum velocity. The root node checks whether acceleration is safe in the next time step, the left child of the root checks whether staying neutral is safe.
-6 and 14 are minimum respectively maximum velocity, and thus have to be treated separately after these two most important splits.

\begin{figure}
	\centering
	\fbox{\includegraphics[width=\textwidth]{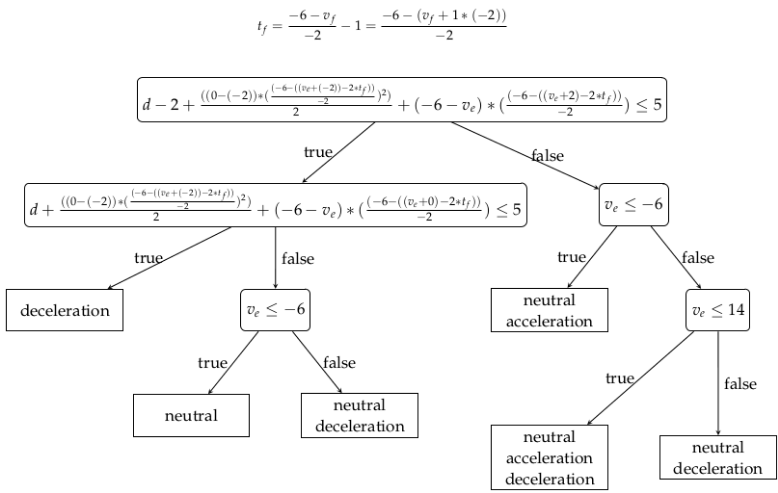}}
	\caption{DT using algebraic predicates to represent the whole permissive controller for the cruise-control model with 11 nodes. }
	\label{fig:SmartCruise}
\end{figure}

\section{Domain knowledge and metadata formats}
\subsection{Metadata file}\label{app:metadatafile}
An example of a metadata file that is used to inform \dtcontrol, which variables are numeric and which are categorical. Giving the column names as well improves the graphical output, as the variables are not indexed (e.g. \texttt{x\_0}), but rather named (e.g. \texttt{Host\_1\_ev}).

\begin{lstlisting}[basicstyle=\ttfamily,columns=flexible,tabsize=4]
{
	"x_column_types": 
	{
		"numeric": [ 0, 1, 2, 3, 4, 5, 6, 7, 8, 9 ],
		"categorical": [ 10, 11, 12, 13, 14, 15, 16	]
	},
	"x_column_names": [
		"Host_1_ev",
		"Host_1_na",
		"Host_1_wt",
		"Host_2_ev",
		"Host_2_na",
		"Host_2_wt",
		"Host_ev",
		"Host_na",
		"Host_wt",
		"N",
		"_loc_Clock",
		"_loc_Host",
		"_loc_Host_1",
		"_loc_Host_2",
		"cr",
		"gave_up",
		"line_seized"
	]
}
\end{lstlisting}

\subsection{Domain knowledge format}\label{app:DK-file}
All domain knowledge must take the form of predicates. Their structure can be summarized as follows:
\[\textit{term} \sim \textit{term}; \textit{def}\]
$\textit{term}$ is an arbitrary arithmetic term, using any elementary function that can be parsed by SymPy, any state-variable and coefficients that are defined in $\textit{def}$; the exact format of the coefficient definition $\textit{def}$ is provided in our documentation, but most importantly, it allows you to specify finite sets of values or a completely arbitrary coefficient.
$\sim$ is a standard comparator from the set \{\texttt{<=,>=,<,>,=}\}.
An example of a predicate is 

\begin{center}
\texttt{c\_1 * x\_1 - c\_2 + 2 * x\_2 <= c\_3; c\_1 in \{1,2,3\}; c\_2 in \{4,8\}}
\end{center}

\noindent Here, $c_1$ and $c_2$ are from a finite set and $c_3$ is completely arbitrary.

\section{Algorithm for categorical predicates}\label{app:avg-alg}
Let $v$ be a categorical state-variable with possible values $v_1, \dots, v_m$.
It is not feasible to simply try all different possible attribute value groupings for every categorical state-variable $v$, since the number of such groupings is exponential in the number of possible values of $v$~\cite[Ch.~7]{Qui93}. We instead use a greedy algorithm based on iterative merging of value groups suggested by~\cite[Ch.~7]{Qui93}, which proceeds as follows:
\begin{enumerate}
	\item Initially, create a single group for every possible value, i.e. set the inital grouping to $G = (\{v_1\},\dots,\{v_m\})$. Let $G_i$ denote the $i$-th set in the grouping.
	\item If only two value groups remain, return those as the optimal grouping.
	\item For every pair of value groups $(G_i, G_j)$, compute the impurity of the new value grouping in which $G_i$ and $G_j$ are merged.
	\item If the impurity has not decreased in any of the new groupings, return the original grouping. Otherwise, proceed to Step 2 with the best new grouping.
\end{enumerate}
Our modification of tolerance that is discussed in Section~\ref{sec:cat} only modifies the stopping condition in Step 4: it replaces ``has not decreased in any of the new groupings'' with ``has increased by more than the tolerance in all new groupings''.

\section{Other impurities}\label{app:impurities}
In the following, we give a description of all impurity measures that are supported by \dtcontrol, and then evaluate them on some models from probabilistic model checking as well as some cyber-physical systems. The text is almost verbatim from the Bachelor's thesis~\cite{MJ}.
\subsection{Description}\label{app:impDescr}
\input{app-impurity-description}
\subsection{Evaluation}\label{app:impEval}
\input{app-impurity-evaluation}

\section{More information on the better determinization}\label{app:MLE}

\input{app-det}

\clearpage

\section{Interactive Interface Screenshot}\label{app:interScreenshot}

\begin{figure}[H]
	\makebox[\textwidth]{\includegraphics[trim={0 14cm 0 0},clip,width=0.75\paperwidth]{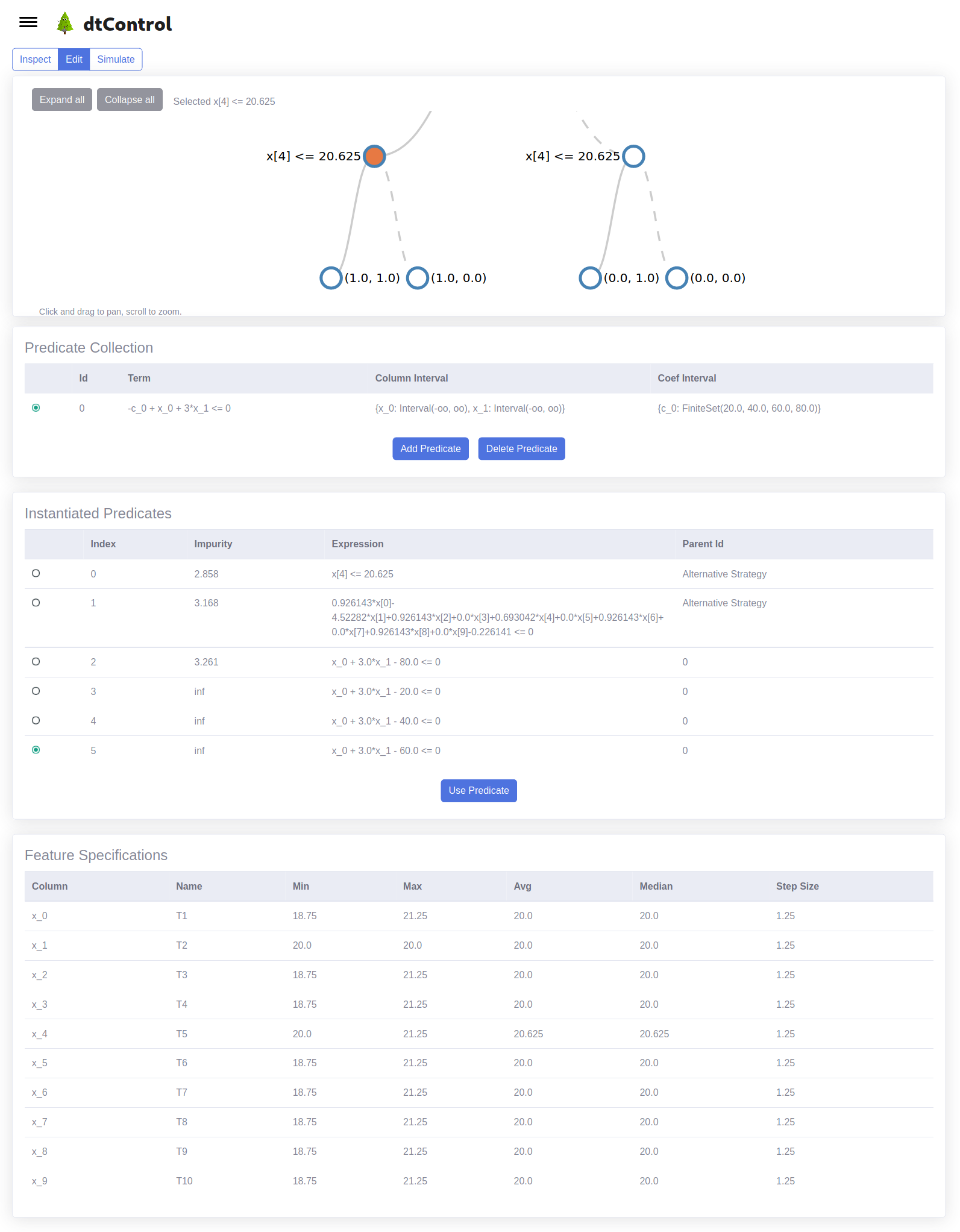}}
	\caption{The interactive interface for predicate selection.}
	\label{fig:interactive}
\end{figure}

%% file: app-impurity-description.tex
\subsubsection{Entropy}
A particularly well-known impurity measure is based on the concept of \textit{entropy} from information theory, as introduced in the seminal work of Shannon~\cite{shannon}. Information theory is concerned with quantifying the amount of information the occurrence of a random event yields. If an event $x$ occurs with probability $p_x$, its information content is defined to be
\[I(x) = -\log_2(p_x)\;\text{bits}. \]

This definition has several desirable properties~\cite[Ch.~3]{Goodfellow2016}: first, we see that the information content of an event is inversely proportional to its probability. For instance, an event with probability $1$ always occurs, and thus does not convey any information. Second, if two events are independent, the information content of both events occurring is the sum of the individual information contents, since
\[I(x,y) = -\log_2(p_{x,y}) = -\log_2(p_xp_y) = -\log_2(p_x) - \log_2(p_y).   \]

In the context of impurity measures, the events that we are interested in are that a randomly picked data point from a sub-controller $C\subseteq S\times 2^A$ allows the action set $B\in 2^A$. Let $n_B$ be the frequency of the set $B$ in $C$, i.e.\ $n_B = |\{ s\in S : C(x)=B \}|$. Then, such an event occurs with probability
\[ \frac{n_B}{|C|} \]
and has an information content of
\[-\log_2\left(\frac{n_B}{|C|}\right)\;\text{bits}. \]

The crucial insight that allows for the development of an impurity measure is the following: if the expected amount of information content of these events is low, we already have a lot of knowledge about the labels in the sub-controller. Thus, classifying this dataset probably requires less effort than classifying a dataset where the expected amount of information content from these events is high. The expected amount of information content is also known as the \textit{entropy} of the controller $C$ and is defined as
\[ H(C) = -\sum_{B\in 2^A} \frac{n_B}{|C|} \log_2\left(\frac{n_B}{|C|}\right). \]

To illustrate two extreme cases, consider a dataset where every data point has a different label. This is extremely hard to classify since every data point has to be separated from all other points, and, correspondingly, the entropy of such a dataset is maximal. In contrast, the entropy of a pure dataset is always 0.

We now have a way to measure the difficulty of classifying a sub-controller. The entropy impurity measure then simply averages the difficulty of classifying the partitions $C_1,\ldots,C_m$ created by a predicate $\rho$. It is thus defined as follows:
\[ \mathrm{ent}(\rho, C) = \sum_{i\in [m]} \frac{ |C_i| }{|C|} H(C_i).\]

Instead of entropy, a similar measure called \textit{information gain} is sometimes used. The only real difference between the two is that information gain is a measure of \textit{goodness}, i.e.\ we want to maximize the information gain in DT learning. This is however equivalent to minimizing entropy~\cite{Qui86}. Entropy and information gain are some of the most common ways to determine the quality of predicates and have received a great deal of attention in the DT literature~\cite{Breiman1984, Qui86, Qui93}.

\subsubsection{Entropy ratio}
\label{sec:ratio}
An issue with entropy that is sometimes encountered with categorical features is that it favors multi-comparison predicates with a large number of branches~\cite{Qui86}. Quinlan~\cite{Qui86} thus introduces a normalization of the information gain criterion called the \textit{gain ratio}. Since this is again a goodness measure, we modified it into an impurity measure named the \textit{entropy ratio}.

We first introduce the quantity of the intrinsic information content of a split
\[ \mathrm{split\;info}(\rho, C) = -\sum_{i\in [m]} \frac{ |C_i| }{|C|} \log_2\left(\frac{ |C_i| }{|C|}\right).\]
This measures the expected information content of the event that a randomly selected data point in $C$ will be assigned the $i^{th}$ branch, which corresponds to the information generated by the partitioning itself. Naturally, the more branches the predicate $\rho$ creates, the higher this information content will be. In contrast, the entropy measures the amount of information \textit{relevant to classification} from the same partitioning~\cite[Ch.~2]{Qui93}.

The entropy ratio then simply normalizes the entropy with the intrinsic information content of a split, i.e.
\[ \mathrm{ent\;ratio}(\rho, C) = \frac{\mathrm{ent}(\rho,C)}{\mathrm{split\;info}(\rho, C)}. \]

It has to be noted that one of the primary reasons for using the entropy ratio instead of just the entropy is to prevent overfitting~\cite[Ch.~2]{Qui93}. In our setting, overfitting is desirable and there is hence less justification for the entropy ratio.

\subsubsection{Gini index}

Another common impurity measure used in e.g.\ the \texttt{CART} system is the \textit{Gini index}~\cite[Ch.~4]{Breiman1984}. It measures the probability of a data point being misclassified if we were to assign labels randomly based on the label distribution in the sub-controller. This probability is given by
\[ G(C) = \sum_{\substack{y,z \in 2^A\\y\neq z}} \frac{n_y}{|C|} \frac{n_z}{|C|}\]
and can equally be written as
\[ G(C) = \left(\sum_{y\in 2^A} \frac{n_y}{|C|}\right)^2 - \sum_{y\in 2^A} \left(\frac{n_y}{|C|}\right)^2 = 1 - \sum_{y\in 2^A} \left(\frac{n_y}{|C|}\right)^2.  \]

Similarly to entropy, the Gini index is then a weighted average of these values:
\[ \mathrm{Gini}(\rho, C) = \sum_{i\in [m]} \frac{ |C_i| }{|C|} G(C_i).\]

\subsubsection{Twoing rule}

The \texttt{CART} system also defines another measure known as the \textit{twoing rule}, which is a goodness measure only defined for binary predicates. It is based on transforming the problem of computing the impurity of a multi-class dataset into the task of computing the impurity of a \textit{two-class} dataset. This transformed problem is then solved with the Gini index defined above.

Let $l$ ($r$) be the number of examples on the left (right) side of the split and $n_{l,y}$ ($n_{r,y}$) be the number of examples with label $y$ on the left (right) side of the split. Then, the twoing rule is defined as follows~\cite[Ch.~4]{Breiman1984}:
\[
\mathrm{twoing}(\rho, C) = \frac{l}{|C|}\frac{r}{|C|}
\left( \sum_{y\in 2^A} \left\lvert \frac{n_{l,y}}{l} - \frac{n_{r,y}}{r}\right\rvert  \right)^2.
\]

Following~\cite{OC1}, the impurity measure we minimize is then simply the reciprocal of the twoing rule.

\subsubsection{Sum minority}

Probably the simplest way to calculate impurity is to count the number of misclassified instances if we were to assign the most frequent label in a partition to all of its data points. This impurity measure is called \textit{sum minority} and due to Heath et al.~\cite{Heath1993}.

Formally, for every partition $C_i$, let $|C_i|$ be the number of examples in that partition and $\gamma(C_i)$ be the label occurring most frequently in $C_i$. Then, define the \textit{minority} $\mu_i$ as
\[ \mu_i = | \{ (x,y)\in C_i : y \neq \gamma(C_i) \} |. \]
The sum minority impurity measure is then simply the sum of these minorities:
\[ \textrm{sum minority}(\rho, C) = \sum_{i \in [m]} \mu_i. \]

\subsubsection{Max minority}

A very similar impurity measure is \textit{max minority}~\cite{Heath1993}, defined as
\[ \textrm{max minority}(\rho, C) = \max_{i \in [m]} \mu_i. \]
It counts the number of misclassified instances in the ``worst'' partition with the maximum number of misclassifications.

Max minority has the theoretical advantage that it produces trees of depth at most $\log_2(|C|)$~\cite{OC1}. However, note that it is the \textit{depth} that is logarithmic in this expression -- the number of \textit{nodes} is linear in $|C|$. Thus, this theoretical insight is not very useful in practice.

\subsubsection{Area under the receiver-operator curve}
\label{def:auroc}
The final impurity measure we discuss has been developed specifically for DTs with linear classifiers in the context of controller representation~\cite{viktor-qest19}. The underlying idea is simple: we want to exploit the knowledge that the controller will be split with a hyperplane, obtained from a linear classifier or a different heuristic. The impurity measure tries to estimate how well separable the controller is by a hyperplane after having been split with the predicate $\rho$.

For now, let us consider the case of only two actions in a controller. In order to estimate how well separable a sub-controller is by a hyperplane, we can again train a linear classifier on this sub-controller and report a metric that measures some quality of this classifier. The simplest such quality would probably be the accuracy, i.e.\ the fraction of data points classified correctly. However, this has the disadvantage that even trivial classifiers that assign the same label to every data point can achieve high accuracy in the case of an imbalanced label distribution. Ashok et al.~\cite{viktor-qest19} instead suggest the usage of the \textit{area under the receiver-operator curve} (AUC), a well-known metric in statistics and machine learning that does not suffer from this drawback.

We thus proceed as follows to estimate the quality of a predicate $\rho$:
\begin{enumerate}
	\item Train a linear classifier for every sub-dataset $C_i$.
	\item Return the sum of the obtained AUC scores of the classifiers.
\end{enumerate}
Note that this is again a goodness measure, which we can transform into an impurity measure by considering the reciprocal.

Ashok et al.~\cite{viktor-qest19} use a different data representation. Instead of mapping states to sets of actions, they map state-action pairs to $\{0,1\}$, indicating whether an action is safe in a state or not. 
Hence they always deal with a binary classification problem. In order to utilize their idea with our different data representation, we again make use of the technique based on one-versus-the-rest classification, cf. \cite[Sec. 4.1.3]{dtcontrol}.
We train one linear classifier for every possible label, which tries to separate this label from the rest, and report a weighted average of the obtained AUC scores.

Note that this impurity measure has the practical disadvantage that we need to train several linear classifiers for every considered predicate, which can be very inefficient.

%% file: app-impurity-evaluation.tex
\subsubsection{Models from probabilistic model checking}

\begin{table}[htbp]
	\centering
	\caption{Number of nodes obtained with different impurity measures and attribute value grouping on models from probabilistic model checking.}
	\label{tab:prism_impurity}
	\begin{tabular}{@{}lrrrrr@{}}
		\toprule
		Case study& Entropy      & Entropy ratio & Gini index   & Sum minority & Max minority \\ \midrule
		csma2\_4       & 48           & 89            & \textbf{43}  & 579          & 1,412         \\\addlinespace[0.2em]
		firewire\_abst & \textbf{9}   & 18            & 12           & 110          & 36           \\\addlinespace[0.2em]
		firewire\_impl & \textbf{77}  & 124           & \textbf{77}  & 442          & 126          \\\addlinespace[0.2em]
		leader4        & 154          & 207           & \textbf{150} & 547          & 1,767         \\\addlinespace[0.2em]
		mer30          & \textbf{158} & 213           & 165          & 7,557         & 6,797         \\\addlinespace[0.2em]
		wlan2          & 222          & 416           & \textbf{220} & 495          & 3,723         \\\addlinespace[0.2em]
		zeroconf       & \textbf{367} & 456           & 374          & 7,243         & 14,624        \\ \bottomrule
	\end{tabular}
\end{table}

Table~\ref{tab:prism_impurity} gives the results of attribute value grouping in combination with different impurity measures on some case studies from probabilistic model checking.
It clearly shows that probabilistic impurity measures such as entropy and gini index perform far better than non-probabilistic impurity measures like sum- and max-minority. Furthermore, we see that the entropy ratio is strictly worse than the standard entropy. As discussed in Section~\ref{sec:ratio}, this is expected, since the main reason for choosing the entropy ratio over just the entropy is normally to prevent overfitting. On the other hand, gini index and entropy perform similarly well and are both viable choices. Note that the twoing rule is not applicable in this scenario, as it is limited to binary predicates.

We also experimented with our modified version of the AUC impurity measure, introduced in Section~\ref{def:auroc}. As previously noted, one of its drawbacks is that it is very expensive to compute. Indeed, it took more than 13 minutes to build a tree with AUC for the small \texttt{firewire\_abst} controller -- in comparison to roughly $0.2$ seconds with the standard impurity measures. On the one hand, this is due to the fact that AUC requires the training of several linear classifiers for every considered predicate, which simply is computationally demanding. On the other hand, we notice that it also produces unnecessarily large trees, which in turn again increases the computational cost: we obtain 716 nodes in the tree for \texttt{firewire\_abst}.

Why does AUC not work in our scenario? There are two factors that come into play: first, the impurity measure tries to estimate the linear separability of the sub-datasets resulting from a predicate. However, since many features in the examples are categorical and we only use oblique splits with numeric features because of explainability reasons, the measure itself is not that meaningful in our context. Second, we conjecture that our approach based on one-versus-the-rest classification, which we adopted due to our data representation, just is not well-suited as an impurity measure. 

\subsubsection{Cyber-physical systems}

For cyber-phyiscal systems (CPS), our experiments suggest that entropy is one of the strongest impurity measures overall in the case of controllers obtained from CPS synthesis. To illustrate, we list the number of nodes when learning DTs with axis-aligned predicates, no determinization, and varying impurity measures in Table~\ref{tab:syn_impurity}.

The table clearly shows that sum- and max-minority perform far worse than the probabilistic impurity measures on many datasets and are overall not competitive. Entropy, gini index, and twoing rule usually perform similarly well, although entropy is slightly better in a number of cases. As expected, the entropy ratio is overall somewhat worse. We again encountered the same performance issues with AUC as before, and the numbers we could compute were not promising, which is why we did not include this impurity measure in the table.%

\begin{table}[htbp]
	\centering
	\caption{Effects of different impurity measures with axis-aligned predicates on decision tree sizes for synthesis of CPS. ``$\infty$'' indicates failure to produce a result within three hours.}
	\label{tab:syn_impurity}
	\setlength{\tabcolsep}{8pt}
	\begin{tabular}{@{}lrrrrrr@{}} 
		\toprule
		Case study    & Entropy         & \bigcell{r}{Entropy\\ratio}                 & \bigcell{r}{Gini\\index} & \bigcell{r}{Sum\\minority}       & \bigcell{r}{Max\\minority}     & \bigcell{r}{Twoing\\rule}                \\ 
		\midrule
		\multicolumn{3}{l}{\textbf{Single-output} }                                   &              &           &         &                       \\
		cartpole      & \textbf{253}    & 257                                         & 255          & 259       & 277     & \textbf{253}          \\\addlinespace[0.2em]
		tworooms      & \textbf{27}     & 37                                          & \textbf{27}  & 39        & 2,627   & \textbf{27}           \\\addlinespace[0.2em]
		helicopter    & \textbf{6,347} & 7,363                                       & 7,177        & 31,835    & 125,727 & 6,429                 \\\addlinespace[0.2em]
		cruise        & \textbf{987}    & 1,161                                       & 1,065        & 11,131    & 89,503  & 1,043                 \\\addlinespace[0.2em]
		dcdc          & \textbf{271}    & 391                                         & 275          & $\infty$  & 2,429   & 277                   \\\addlinespace[0.2em]
		\multicolumn{3}{l}{\textbf{Multi-output} }                                    &              &           &         & \multicolumn{1}{l}{}  \\
		tenrooms      & 17,297          & \textbf{15,951 }                            & 17,297       & 18,565    & 26,751  & 17,415                \\\addlinespace[0.2em]
		truck\_trailer & 338,389         & 348,959 &   \textbf{312,741}           &   442,013        &    561,083     &       316,457                \\\addlinespace[0.2em]
		traffic       & \textbf{12,573}          &       $\infty$                                      & 16,627             &     276,067      & $\infty$        &   15,319                    \\\addlinespace[0.2em]
		vehicle       & 13,237          & 15,677                                      & 13,135       & 32,271    & 39,129  & \textbf{13,109}       \\\addlinespace[0.2em]
		aircraft      & \textbf{913,857}         &      932,625                                       &   923,709           &   $\infty$        &     2,242,773    &      922,727                 \\
		\bottomrule
	\end{tabular}
\end{table}

%% file: app-det.tex
Here we give more information on the new insights about determinization.
First we give the proof that it suffices to consider complete determinization functions. 
Then we give the derivation of multi-label entropy and afterwards multi-label Gini index.
Lastly, we give another derivation of multi-label entropy and Gini-index, to provide more intuition and insight to their workings.
Parts of this appendix are almost verbatim from the Bachelor's thesis~\cite{MJ}.

\subsection{Considering complete determinization functions suffices}\label{app:completeDet}
We show that we can reduce the search space by only considering complete determinization functions, i.e. determinization functions such that for all states $s$ we have $\abs{\detfun(C)(s)} = 1$.
In particular, we prove that there always exists a complete determinization with minimal impurity in Proposition~\ref{prop:completeness} and Theorem~\ref{theo:completeness}. We limit our discussion to the entropy and the Gini index, since these impurity measures are the most widely used and performed the best in our experiments.
We shorten the terminology and in the following refer to determinization functions as determinizations.

\begin{proposition}
	\label{prop:completeness}
	Given an impurity measure $\phi\in\{\mathrm{ent}, \mathrm{Gini}\}$, a predicate $\rho$, and a controller $D$, for every incomplete determinization $\delta$ of $D$ there exists a complete determinization $\bar{\delta}$ of $D$ with $\phi(\rho, \bar\delta, D) \leq \phi(\rho, \delta, D)$.
\end{proposition}
\begin{proof}
	We outline a procedure that turns an incomplete determinization into a complete determinization with at most the same impurity. Let $\delta$ be an incomplete determinization of $D$ with co-domain $\im$. Furthermore, let $n_{i,y}$ denote the number of data points that are assigned the (possibly non-deterministic) label $y\in\im$ under $\delta$ in the $i^{th}$ sub-dataset $D_i$ created by $\rho$.
	
	Let us start with $\phi = \mathrm{ent}$. We have that
	\[ \mathrm{ent}(\rho, \delta, D) = \sum_{i\in [m]} \frac{ |D_i| }{|D|} H(\delta, D_i),\]
	where
	\begin{gather*}
	H(\delta, D_i) = -\sum_{y\in \im}\frac{n_{i,y}}{|D_i|}\log_2\left(\frac{n_{i,y}}{|D_i|}\right).
	\end{gather*}
	
	Since $\delta$ is incomplete, there is a label $q\in \im$ with $|q| > 1$. Consider the determinization $\bar{\delta}$ that is equivalent to $\delta$, except that it assigns the single-label $r\in q$ to all data points $x\in X$ where $\delta(x) = q$. We define $\bar{n}_{i,y}$ for $\bar{\delta}$ equivalently to $n_{i,y}$ for $\delta$. We fix a specific sub-dataset $D_i$ and drop the corresponding index to simplify notation. Then,
	\begin{align*}
	H\left(\bar{\delta}\right) &= -\sum_{y\in \imbar}\frac{\bar{n}_y}{|D|}\log_2\left(\frac{\bar{n}_y}{|D|}\right)\\
	&= -\sum_{y\in \im \setminus\{q,\, r\}}\frac{n_y}{|D|}\log_2\left(\frac{n_y}{|D|}\right) - \frac{\bar{n}_r}{|D|}\log_2\left(\frac{\bar{n}_r}{|D|}\right).
	\end{align*} 
	It follows that
	\begin{gather*}
	H(\delta) - H\left(\bar{\delta}\right)\\
	= -\frac{n_q}{|D|}\log_2\left(\frac{n_q}{|D|}\right) - \frac{n_{r}}{|D|}\log_2\left(\frac{n_{r}}{|D|}\right) + \frac{\bar{n}_r}{|D|}\log_2\left(\frac{\bar{n}_r}{|D|}\right).
	\end{gather*}
	
	With
	\[\bar{n}_r = n_r + n_q,\]
	we obtain:
	\begin{gather*}
	H(\delta) - H\left(\bar{\delta}\right)\\
	= -\frac{n_q}{|D|}\log_2\left(\frac{n_q}{|D|}\right) - \frac{n_{r}}{|D|}\log_2\left(\frac{n_{r}}{|D|}\right) +
	\frac{n_r + n_q}{|D|}\log_2\left(\frac{n_r + n_q}{|D|}\right).
	\end{gather*}
	First, consider the special case of $n_r = 0$. As is usual in information theory, we evaluate $0\log_2(0)$ as $\lim_{x\rightarrow0}x\log_2(x) = 0$, and thus arrive at
	\begin{gather*}
	H(\delta) - H\left(\bar{\delta}\right) = 0.
	\end{gather*}
	If $n_r > 0$, we have
	\begin{gather*}
	H(\delta) - H\left(\bar{\delta}\right)\\
	= \frac{n_q}{|D|}\left(\log_2\left(\frac{n_r + n_q}{|D|}\right)-\log_2\left(\frac{n_q}{|D|}\right)\right) + \frac{n_r}{|D|}\left(\log_2\left(\frac{n_r + n_q}{|D|}\right)-\log_2\left(\frac{n_r}{|D|}\right)\right)\\
	> 0,
	\end{gather*}
	where the last step follows from the fact that $\log_2$ is strictly increasing.
	
	Thus, for every sub-dataset $D_i$, we have $H\left(\bar{\delta}, D_i\right) \leq H(\delta, D_i)$, and consequently
	\[ \mathrm{ent}(\rho, \bar{\delta}, D) \leq  \mathrm{ent}(\rho, \delta, D).\]
	
	Note the following key point: $\bar{\delta}$ is ``more deterministic'' than $\delta$ as there are fewer states to which it assigns a label $y$ with $|y| > 1$. If we thus continue this process of producing ``more deterministic'' determinizations (now starting with $\bar{\delta}$), we will eventually reach a complete determinization with an entropy less than or equal to the entropy of $\delta$.
	
	A similar analysis can be conducted for the case of $\phi = \mathrm{Gini}$. With the same definitions as above, we obtain
	\[
	G\left(\bar\delta\right) = 1 - \sum_{y\in \im \setminus \{ q,\,r \}} \left(\frac{n_y}{|D|}\right)^2 - \left(\frac{\bar{n}_r}{|D|}\right)^2.
	\]
	Then,
	\begin{gather*}
	G(\delta) - G\left(\bar{\delta}\right)\\
	= -\left(\frac{n_q}{|D|}\right)^2 -\left(\frac{n_r}{|D|}\right)^2
	+\left(\frac{n_r + n_q}{|D|}\right)^2\\
	= -\frac{n_q^2}{{|D|}^2} -\frac{n_r^2}{{|D|}^2}
	+\frac{n_r^2 + 2n_rn_q + n_q^2}{{|D|}^2}\\
	= \frac{2n_rn_q}{{|D|}^2}\\
	\geq 0.
	\end{gather*}
	
	Therefore, similar as above, we have
	\[\mathrm{Gini}(\rho, \bar\delta, D) \leq \mathrm{Gini}(\rho, \delta, D)\]
	and can continue this process to eventually reach a complete determinization with Gini index less than or equal to the Gini index of $\delta$.
\end{proof}

\begin{theorem}
	\label{theo:completeness}
	Let $\Delta^*$ be the set of determinizations that achieve the minimal impurity with respect to an impurity measure $\phi \in \{ \mathrm{ent},\mathrm{Gini} \}$, a predicate $\rho$, and a dataset $D$. Then, there exists a $\delta^* \in \Delta^*$ that is complete.
\end{theorem}
\begin{proof}
	We give an indirect proof. Assume every determinization $\delta^* \in \Delta^*$ is incomplete. Then, by Proposition~\ref{prop:completeness}, we know that there exists a $\bar\delta^*$ that is complete and $\phi(\rho, \bar\delta^*, D) \leq \phi(\rho, \delta^*, D)$ for every $\delta^* \in \Delta^*$. However, this means that $\bar\delta^*$ achieves the minimal impurity and would have to be an element of $\Delta^*$. Therefore, $\Delta^*$ cannot be the set of determinizations that achieve minimal impurity.
\end{proof}

Theorem~\ref{theo:completeness} shows that it suffices to consider all complete determinizations to determine the best predicate. However, the number of complete determinizations is still far too large to simply enumerate them all. 

\subsection{Multi-label entropy derivation}

We start our derivation with the following general formulation for multi-label entropy of a controller $C$ and a determinization $\detfun$. For a controller $C$, $\image(C)$ denotes the codomain and $\abs{C}$ denotes the size of the domain, and $p_y$ is the empirical frequency of $y$ in $\detfun(C)$, formally $p_y = \abs{\{s \in \detfun(C) \mid \detfun(C)(s) = y\}}$:

\[H(C,\detfun) = - \sum_{y \in \image(\detfun(C))} \frac{p_y}{\abs{C}} \log_2(\frac{p_y}{\abs{C}})\]

If $\detfun$ is the identity-function, the considered co-domain is $2^A$, and we have the classic entropy. 
For any other determinization function, we get different $p_y$ and thus different impurities.
As mentioned earlier, trying every possible determinization function $\detfun$ and calculating the precise $p_y$ is optimal, but infeasible.
We showed in Theorem~\ref{theo:completeness} that for calculating impurities it suffices to consider determinization functions that map to singleton sets, i.e. for all states $s$ we have $\abs{\detfun(C)(s)} = 1$.
Then $y$ is a singleton set $\{a\}$ for some action $a \in A$, and $p_y$ is the frequency of $\{a\}$ in the determinized controller. 
Thus, we can over-approximate $p_y$ by counting the occurrences of every action $a$ in the controller, which is feasible and was already done for MaxFreq. 
The difference is that MaxFreq used this to explicitly calculate one fixed determinization function $\detfun$ that was used for all predicates, while the new approach uses it to over-approximate the multi-label impurity, implicitly using a different determinization function for every considered predicate\footnote{Note that we formulate the multi-label impurity for a single sub-controller. As for all previous impurity measures, the impurity of a predicate then is the weighted average of the impurities of the sub-controllers. Thus, since the determinization comes into play after the split in every sub-controller, a different determinization is considered for every predicate.}.

To complete the formulation of our heuristic, we add that in the corner case that all states agree on an action, the entropy should be 0, and thus get the following impurity measure, where for an action $a\in A$ we use $p_a$ to denote the frequency of that action in the un-determinized controller, i.e. $p_a =  \abs{\{s \in (C) \mid a \in (C)(s)\}}$:
\[\mle(C) = \begin{cases}
0 &\mbox{if } \exists a \in A, \forall s \in S: a \in C(s)\\
- \sum_{a \in A} \frac{p_a}{\abs{C}} \log_2(\frac{p_a}{\abs{C}}) &\mbox{otherwise}
\end{cases}\]

\subsection{Multi-label Gini index derivation}\label{app:MLGini}

We proceed with the multi-label formulation of the Gini index, which can be derived similar to multi-label entropy. Applying Theorem~\ref{theo:completeness}, we obtain that
\begin{equation}
\label{eq:mlGini}
G(\delta_\rho^*, C) = 1 - \sum_{l\in L} \left(\frac{n_{i,l}}{|C|}\right)^2.
\end{equation}
We can again estimate the $n_{i,l}$ with the approximation of maximum frequency as in the entropy derivation. However, we need to be careful since we over-approximate and the value of the sum in Eq.~\ref{eq:mlGini} can therefore be greater than 1. In order to keep the impurity non-negative, we thus need to subtract the sum not from 1, but from its maximal value $|L|$. Finally, again treating the corner case of all labels agreeing, we get
\[
\widehat{G}(C) =
\begin{dcases}
0,& \text{if } \exists l\in L: f_i(l) = |C|\\
|L|-\sum_{l\in L}\left(\frac{f_i(l)}{|C|}\right)^2
,& \text{otherwise}.
\end{dcases}
\]
The complete multi-label Gini index is then again the weighted average of the estimated values $\widehat G(C)$.

\subsection{An alternative point of view}\label{app:MLintuition}
Having derived multi-label impurity measures formally, we now want to point out an alternative, more intuitive point of view that may shed some light on their internal workings. Let us fix a specific sub-controller $D_i$ with frequency function $f_i$. Plotting $f_i / |D_i|$, i.e.\ the fraction of data points that can be assigned a specific single-label, yields a bar chart that may look like the one depicted in Fig.~\ref{fig:barchart}~(left).

\begin{figure}
	\centering
	\subfloat[]{\includegraphics[width=0.4 \textwidth]{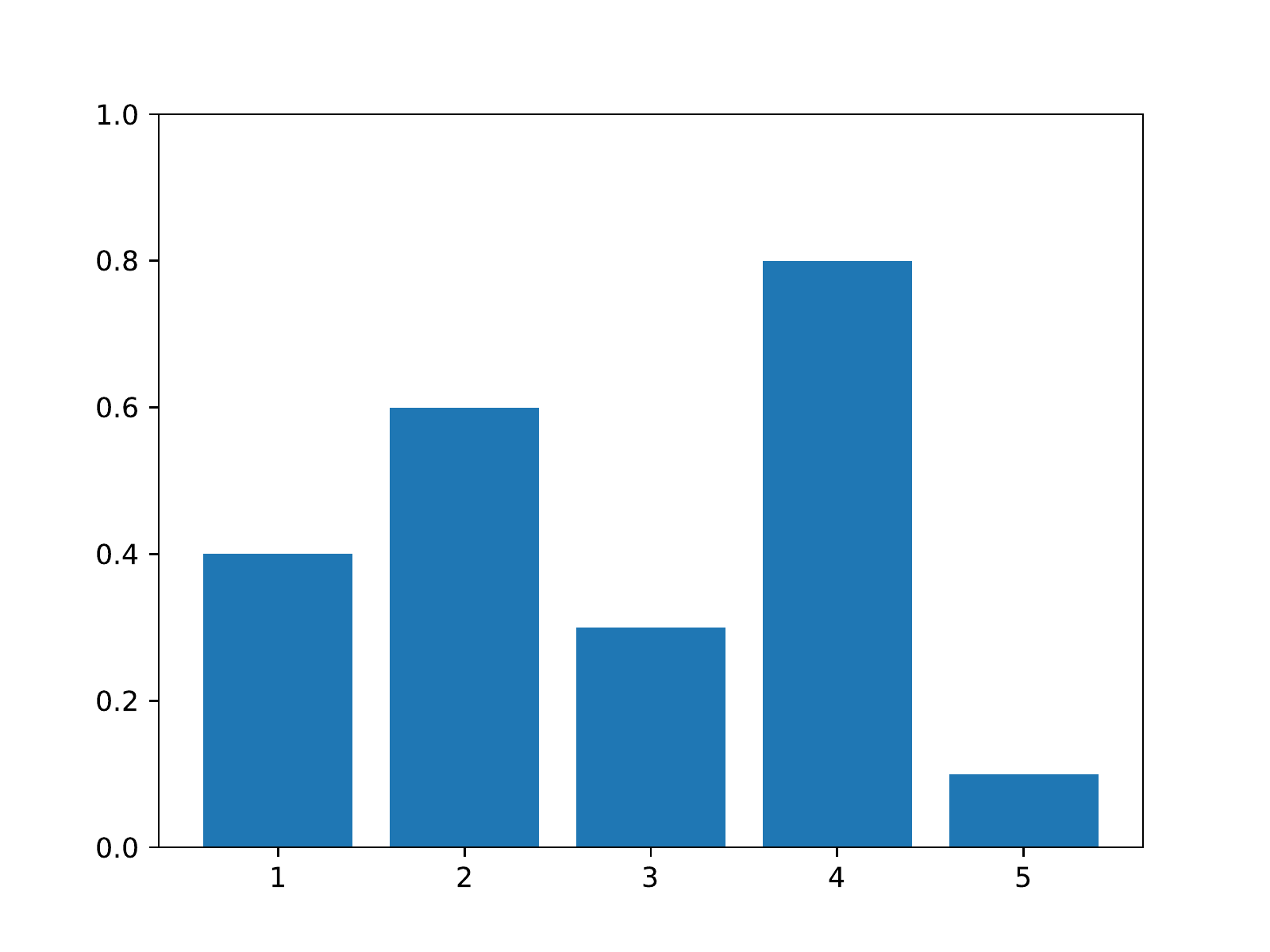}}
	\qquad
	\subfloat[]{\includegraphics[width=0.4\textwidth]{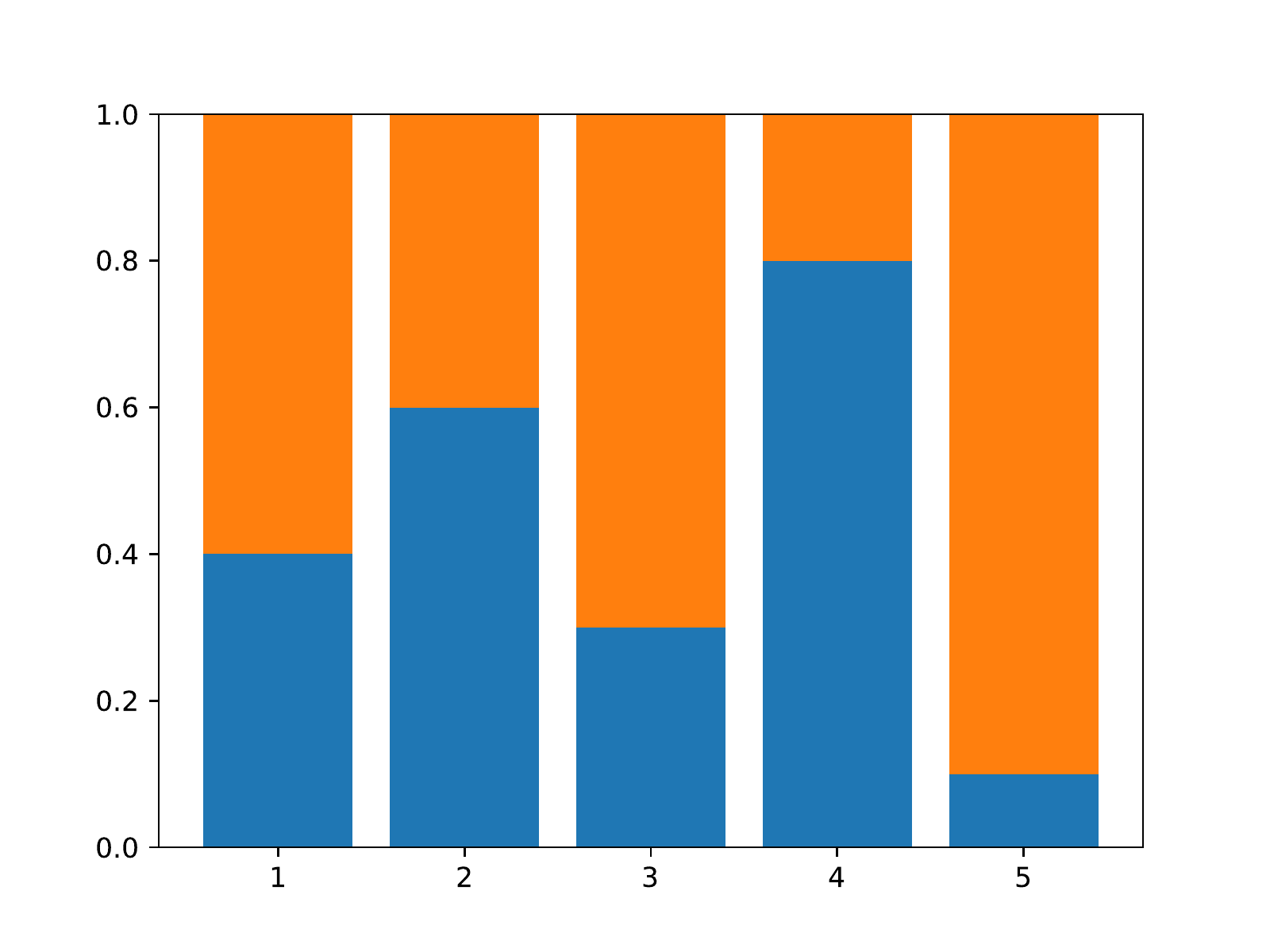}}
	\caption[Bar chart of the frequency function]{Bar chart of the frequency function. The bar chart (left) for a controller with actions $1, \ldots, 5$ and the error bars indicating the impurity (right).}
	\label{fig:barchart}
\end{figure}

If we now had to construct an impurity measure solely from this bar chart, we might make the following observations:
\begin{enumerate}
	\item The impurity should be 0 if all bars have a value of 1, since then every label can be assigned to every data point.
	\item The impurity should be high if there are many bars with low values.
	\item Generally, if there are fewer bars the impurity should be lower, because if there are fewer labels we will probably need fewer splits of the dataset.
\end{enumerate}

This already rules out two simple ideas that come to mind: we cannot use the reciprocal of the sum of the bars, because this would violate point 2 if there is a great number of different labels. We also cannot use the reciprocal of the mean of the bars, since this would not take observation 3 into account. Instead, we could come up with the following impurity measure that satisfies all three desired properties: we measure how much is missing from each bar to get a value of 1 and return the sum of these values. This idea is depicted in Fig.~\ref{fig:barchart}~(right).

Formalizing this concept yields the following function to measure the impurity of a sub-dataset:
\begin{align}
J(D_i) &= \sum_{l\in 2^A} 1 - \frac{f_i(l)}{|D_i|}\\
&= |L| - \sum_{l\in 2^A} \frac{f_i(l)}{|D_i|} \label{eq:bin}.
\end{align}
We could then compute the impurity of a predicate as the weighted average of the values $J(D_i)$ for every sub-dataset $D_i$.

Eq.~\ref{eq:bin} already looks surprisingly similar to the function $\widehat G(D_i)$ of the multi-label Gini index. Indeed, we see that $\widehat G(D_i)$ is merely a scaled version of $J(D_i)$: before computing the error bars, the bar chart is scaled with the function $s(x) = x^2$, which penalizes smaller bars more strongly.

With the same idea we can also work towards the multi-label entropy. Consider the scaling function $s(x) = 1 + x\log_2(x)$. We have
\begin{align*}
J(D_i, s) &= \sum_{l\in L} 1 - s\left(\frac{f_i(l)}{|D_i|}\right)\\
&= \sum_{l\in L} 1 - 1 - \frac{f_i(l)}{|D_i|}\log_2\left(\frac{f_i(l)}{|D_i|}\right)\\
&= -\sum_{l\in L}\frac{f_i(l)}{|D_i|}\log_2\left(\frac{f_i(l)}{|D_i|}\right),
\end{align*}
which matches the function $\widehat H(D_i)$ of the multi-label entropy.

The scaling functions of the multi-label entropy and Gini index are plotted in Fig.~\ref{fig:scaling}. As previously mentioned, the Gini index scaling function especially increases the impurity assigned to small bars. In contrast, the entropy scaling function penalizes bars with a value of approximately 0.37 the most heavily and assigns small impurity to bars with very low value. While not quite as intuitive at first glance, this also seems like a valid approach: small bars mean that only few data points can be assigned a particular label and we thus only have to separate those few data points. On the other hand, a label that can be assigned to around 40 percent of the examples means that we might have to split off a large fraction of the dataset.

\begin{figure}
	\centering
	\subfloat[]{\input{figures/scaling.tex}}
	\qquad
	\subfloat[]{\input{figures/scaling2.tex}}
	\caption[Scaling functions used in multi-label impurity measures]{Plots of the scaling functions arising out of multi-label entropy (left) and Gini index (right).}
	\label{fig:scaling}
\end{figure}
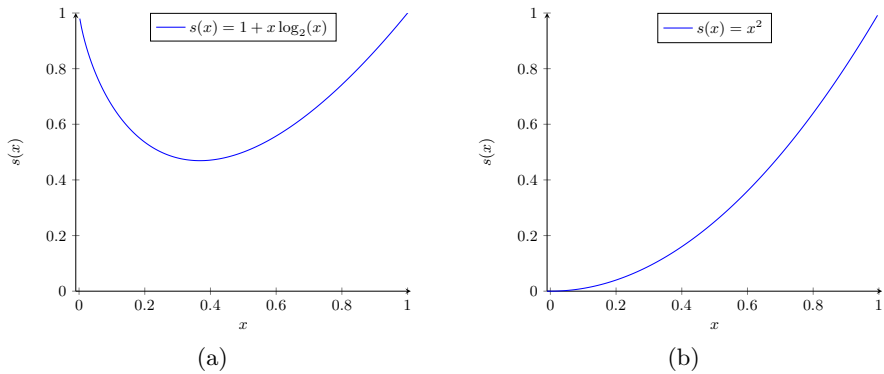

%% file: figures/scaling.tex
\begin{tikzpicture}[scale=.65]
	\begin{axis}[axis lines = left,xlabel = $x$, ylabel = {$s(x)$}, ymin=0, ymax=1, xmin=-.01, xmax=1.01, legend style={at={(0.5,1)},anchor=north}]
		\addplot [samples=2000,color=blue, style=semithick]
		{1+x*log2(x)};
		\addlegendentry{$s(x) = 1+x\log_2(x)$}
	\end{axis}
\end{tikzpicture}

%% file: figures/scaling2.tex
\begin{tikzpicture}[scale=.65]
	\begin{axis}[axis lines = left,xlabel = $x$, ylabel = {$s(x)$}, ymin=0, ymax=1, xmin=-.01, xmax=1.01, restrict y to domain=0:1, legend style={at={(0.5,1)},anchor=north}]
		\addplot [samples=1000,color=blue, style=semithick]
		{x^2};
		\addlegendentry{$s(x) = x^2$}
	\end{axis}
\end{tikzpicture}